\documentclass[11pt]{article}
\usepackage[T1]{fontenc}         
\usepackage{amsfonts}      
\usepackage{nicefrac}      
\usepackage{microtype}   
\usepackage[colorlinks,
            linkcolor=red,
            anchorcolor=blue,
            citecolor=blue, 
            pagebackref=true
           ]{hyperref}
\usepackage{fullpage}
\usepackage{tabularx}
\usepackage{float}
\usepackage{enumitem}
\usepackage{wrapfig,lipsum}
\usepackage{algorithm}
\usepackage{verbatim}
\usepackage{hyperref}

\usepackage{mmll}
\usepackage{mathtools}
\usepackage{caption}
\usepackage{nicematrix}

\usepackage[colorinlistoftodos, textsize=tiny, textwidth=2.2cm, backgroundcolor=blue!10, linecolor=magenta, bordercolor=magenta]{todonotes}

\renewcommand*{\backref}[1]{}
\renewcommand*{\backrefalt}[4]{%
\ifcase #1 %
    No citations.%
\or
    (p. #2.)%
\else
    (pp. #2.)%
\fi}%
\allowdisplaybreaks
\title{Inference-Time Alignment of Diffusion Models via Trust-Region Iterative Twisted Sequential Monte Carlo}

\author{%
    Weixin Wang\thanks{ 
    Duke University; email: {\tt weixin.wang@duke.edu}}~\footnotemark[3]
    ~~
    Yu Yang\thanks{ 
    Duke University; email: {\tt yu.yang@duke.edu}}~\thanks{Equal contribution} 
    ~~
    Wei Deng\thanks{
    Morgan Stanley; email: {\tt
    weideng056@gmail.com}}~\footnotemark[6]
    ~~
    Pan Xu\thanks{Duke University; email: {\tt pan.xu@duke.edu}}~\thanks{Corresponding authors} 
}

\begin{document}
\maketitle

\begin{abstract}
We study inference-time alignment for diffusion-based generative models, aiming to steer a base model toward high-reward outputs without updating its weights. Recent Sequential Monte Carlo (SMC)-based steering methods approximate reward-tilted target distributions in a principled way, but their proposals remain largely tied to the base sampler. Since reward information is mainly used after propagation through particle reweighting and resampling, these methods can require large particle budgets and suffer from weight degeneracy and high-variance estimates. One way to reduce variance and improve particle efficiency is to iteratively learn twisting functions that provide look-ahead guidance, as in twisted SMC. However, existing learnable twisting methods are developed mainly for classical sequential inference and can be unstable when applied to diffusion-based alignment with high-dimensional state spaces and terminal, noisy, or black-box rewards. We propose Trust-Region Iterative Twisted Sequential Monte Carlo (\textbf{TRI-TSMC}), a trust-region framework for learning twisting functions in SMC-based inference-time alignment. Each iteration computes an exact KL-constrained update in path space, which admits a closed-form solution by tempered importance reweighting, and projects this target back to the parameterized twisted family by weighted maximum likelihood. Theoretically, we formalize the value-function interpretation of the optimal twisting function and show that it yields a zero-variance sampler. We prove that the trust-region update follows an escort path toward the target distribution, that the weighted maximum-likelihood update is a forward-KL projection, and that the path reduces residual importance-weight variance. Empirically, TRI-TSMC improves primary alignment objectives on discrete diffusion text generation and text-to-image generation under matched inference-time budgets.
\end{abstract}

\section{Introduction}

Aligning generative models with human preferences or task-specific reward functions is important for practical deployment. In this paper, we focus on diffusion-based generative models, including continuous diffusion models and discrete diffusion language models. Although these models can generate high-quality images or text, they can still produce outputs that deviate from user intentions or task-specific requirements. Existing alignment methods can be broadly grouped into three categories: (1) post-training fine-tuning methods \citep{black2024training, wallace2024diffusion,clark2023directly,fan2023dpok}; (2) inference-time methods \citep{dhariwal2021diffusion, chung2022diffusion}; (3) distillation-based hybrid approaches \citep{uehara2025inference}. Post-training methods, such as DDPPO and Diffusion-DPO \citep{black2024training, wallace2024diffusion}, update the parameters of the pretrained base model to shift the learned sampling distribution toward high-reward outputs. These methods can produce efficient aligned samplers after training, but they require access to model weights, reward-specific training data, and substantial training computation. They may also reduce sample diversity through alignment tax or mode collapse \citep{stiennon2020learning, bai2022training, fan2023dpok}. More importantly, when the reward model or target preference changes, post-training methods usually require another round of optimization. Distillation methods reduce inference cost by training a student model to imitate an aligned sampler, but they still require an additional training stage and depend on the quality of the teacher sampler \citep{uehara2025inference}. 
Inference-time alignment is more flexible because it keeps the pretrained base model fixed and only changes the sampling procedure. Gradient-based guidance methods, such as Classifier Guidance \citep{dhariwal2021diffusion} and Diffusion Posterior Sampling \citep{chung2022diffusion}, use local gradients from an external signal to adjust each denoising step. However, many alignment rewards are terminal, non-differentiable, or provided by black-box evaluators, such as preference models or image reward models. In these settings, there is no reliable local gradient at every intermediate denoising state. 
SMC-based steering methods address this issue by formulating alignment as sampling from a reward-tilted target distribution \citep{singhal2025general, wu2023practical, uehara2025inference}. They can handle terminal or black-box rewards through particles, intermediate potentials, importance weighting, and resampling. However, their proposal dynamics usually remain largely tied to the sampler induced by the pretrained base model. Reward information is mainly used after particles have already been propagated, through reweighting and resampling. When high-reward trajectories are rare under the base sampler, many particles lie in low-reward regions before resampling. This leads to weight degeneracy, high-variance estimates, and a strong dependence on large particle budgets. 

Twisted SMC provides a natural way to reduce variance and improve particle efficiency by modifying the proposal itself  \citep{guarniero2017iterated}. A twisting function gives look-ahead guidance: states that are expected to lead to large future potential receive larger proposal probability. Theoretically, with the optimal twisting function, the residual importance weights become constant, giving a zero-variance importance sampler \citep{heng2020controlled, guarniero2017iterated}. Therefore, the key question is how to iteratively  learn a good twisting function for high-dimensional generative alignment. 
Existing learnable twisting methods, such as iAPF \citep{guarniero2017iterated} and TPPF \citep{lu2024guidancetwistedparticlefilter}, were mainly developed for classical sequential statistical inference. Extending them to diffusion-based generative alignment is not straightforward. Sequential statistical inference problems often involve lower-dimensional states and structured observation models, while diffusion alignment involves high-dimensional image, latent, or token states, neural reverse transitions, and rewards that are often terminal, black-box, or estimated from partially denoised samples. In this regime, recursive regression targets can be noisy, and direct path-space optimization can have high variance especially when the current proposal is far from the target.

Our method is motivated by the value-function interpretation of the optimal twisting function. After a logarithmic transformation, the optimal twisting recursion becomes a soft Bellman recursion under the Feynman--Kac formulation. This connection suggests that twisting-function learning should be stabilized by conservative local updates, in the same spirit as trust-region methods for value and policy learning in RL \citep{schulman2015trust} and stochastic control problem \citep{blessing2025trust}. 
We therefore propose \textbf{Trust-Region Iterative Twisted Sequential Monte Carlo (TRI-TSMC)}, a trust-region framework for iteratively learning twisting functions in SMC-based inference-time alignment. Each iteration first computes an exact KL-constrained update in path space, which has a closed-form solution by tempered importance reweighting. It then projects this intermediate target back to the parameterized twisted family by weighted maximum likelihood. \textbf{Our contributions} are summarized as follows:
\begin{itemize}[leftmargin=10pt,nosep]
    \item We propose TRI-TSMC, a trust-region framework for iteratively learning twisting functions in SMC-based inference-time alignment of diffusion-based generative models. The method alternates between an exact KL-constrained path-space update and a weighted maximum-likelihood projection onto the parameterized twisted family.

    \item We provide a theoretical analysis of TRI-TSMC. We formalize the value-function interpretation of the optimal twisting function, show its zero-variance property, and prove that the trust-region update follows an escort path while reducing residual importance-weight variance.

    \item We evaluate TRI-TSMC on discrete diffusion text generation and text-to-image diffusion alignment. TRI-TSMC improves the primary alignment objectives over inference-time baselines under matched inference-time budgets, with additional analyses showing the quality--diversity trade-off.
\end{itemize}

\section{Related Work}

\paragraph{Diffusion Model Alignment.}

Diffusion alignment studies steering a pretrained diffusion model toward a reward-tilted target distribution of the form $p^{\text{target}}(x) \propto p^{\text{pre}}(x)\exp(r(x)/\alpha)$, where $r(x)$ is a task-specific reward, possibly provided by a black-box evaluator or a learned reward model, and $\alpha>0$ is a temperature parameter. Existing methods can be broadly grouped into three directions: post-training, inference-time alignment, and hybrid distillation. 

Post-training methods directly update model parameters so that the learned reverse process shifts toward high-reward samples while staying close to the pretrained model \citep{lee2023aligning}. \citet{black2024training} formulates diffusion denoising as a sequential decision process and applies policy-gradient fine-tuning. Other methods directly backpropagate through the denoising chain when the reward is differentiable \citep{clark2023directly}, or introduce preference-based objectives and KL regularization to improve stability and reduce drift from the base model \citep{fan2023dpok, wallace2024diffusion}. These methods can produce efficient aligned samplers after training, but they require updating model weights and are usually tied to a specific reward.

Inference-time alignment keeps the pretrained model fixed and modifies the sampling procedure instead. This line of work is attractive when rewards are non-differentiable, expensive to evaluate, or frequently changed across tasks. A classical direction is gradient-based guidance, where each reverse step is adjusted using gradients from an external signal. Classifier guidance \citep{dhariwal2021diffusion} is the standard example, and later work extends this idea to more general posterior sampling and inverse problems \citep{chung2022diffusion}. Another direction uses particle-based methods such as Sequential Monte Carlo to approximate reward-tilted target distributions. These methods are especially useful when guidance is only available through a terminal reward or a black-box evaluator, so there is no simple gradient signal at each step. Recent work has developed several SMC-based approaches for inference-time alignment in continuous diffusion models and discrete diffusion language models. FK-Steering formulates reward-guided diffusion generation as a Feynman--Kac particle system and shows that increasing the inference-time particle budget can improve alignment quality without updating the base model \citep{singhal2025general}. Twisted Diffusion Sampler (TDS) similarly uses SMC for diffusion models, with a focus on practical and asymptotically exact conditional sampling \citep{wu2023practical}. Beyond standard particle steering, optimized auxiliary particle filters study how to adapt proposal mixtures through convex optimization \citep{branchini2021optimized}, while test-time SMC alignment for discrete diffusion constructs training-free local proposals using Taylor and Gumbel--Softmax approximations \citep{pani2025test}. For masked diffusion language models, ILRR performs reference-based representation steering \citep{avrahami2026ilrr}, whereas Self-Rewarding SMC uses trajectory-level model confidence as an intrinsic signal for particle reweighting and resampling \citep{luo2026self}. 

A third direction combines inference-time alignment and post-training through distillation. The main idea is to first run a stronger inference-time method as a teacher, and then train a student model to imitate the guided trajectories or conditional transitions. This can convert a slow but accurate inference-time sampler into a fast aligned generator. Progressive distillation was originally proposed for acceleration \citep{salimans2022progressive}, and recent work discusses policy distillation more broadly as a way to transfer guidance, SMC, or value-based steering back into model weights \citep{uehara2025inference}. Related to these directions, several recent papers study diffusion alignment from an optimal-control perspective and connect aligned sampling to Doob's $h$-transform or stochastic control formulations \citep{domingo2024adjoint, denker2024deft, so2026discrete}. This view is closely related to the value-function interpretation of twisting functions used in our method.

\section{Preliminary}

\label{sec:preliminary}

\subsection{Inference-time Alignment as a Unified Sequential Sampling Problem}

We study inference-time alignment for diffusion-based generative models, including continuous-state diffusion models and discrete diffusion language models, where generation is fundamentally viewed as a sequential stochastic process over a trajectory $X_{0:T}=(X_0,X_1,\dots,X_T)$, where $X_t\in\mathcal{X}_t$ and $\mathcal{X}_t$ denotes the respective state space. For continuous diffusion models, $\mathcal{X}_t$ is a continuous image or latent space; conversely, for discrete diffusion language models, $\mathcal{X}_t$ is a discrete token or mask-state space. The trajectory $X_{0:T}$ represents the full path generated by the reverse denoising process. To maintain notational consistency throughout this paper, we index the process by the sampler iteration rather than the physical noise level; thus, $X_0$ denotes the initial state of the sampler and $X_T$ denotes the final state at which the output is realized. The pretrained base model induces a base path measure $P(dx_{0:T}) = \mu(dx_0) \prod_{t=1}^T f_t(dx_t \mid x_{t-1})$, where $\mu$ is the initial distribution and $f_t$ is the transition kernel from $\mathcal{X}_{t-1}$ to $\mathcal{X}_t$. Our method leverages the sequential nature of generation.
 
Inference-time alignment keeps the pretrained base model fixed and modifies only the sampling procedure at inference time. The fixed base model induces the base path measure $P$, which serves as the reference distribution for alignment. Given a reward function $r:\mathcal{X}_T\to\mathbb{R}$ that measures a target attribute or human preference of the final output, the goal is to sample from a reward-tilted path distribution. Following \citet{singhal2025general, uehara2025inference}, we define the target alignment distribution as
\begin{align}
\label{equ:alignment_target}
    \pi(dx_{0:T}) \propto P(dx_{0:T}) \exp\bigg(\frac{r(x_T)}{\alpha}\bigg),
\end{align}
where $\alpha>0$ is the temperature parameter, $x_T$ denotes the final generated object, such as final clean sample in continuous diffusion or final token sequence in discrete diffusion language models.

A key difficulty is that the reward in \eqref{equ:alignment_target} is often terminal: it is evaluated only on the final output $X_T$, while intermediate states usually do not have a natural stepwise guidance signal. As a result, direct sampling from $\pi$ is difficult in long-horizon or high-dimensional generation, since most partial trajectories provide little information about whether they will eventually lead to a high-reward output.

\subsection{Sequential Monte Carlo for Terminal-Reward Alignment}

To handle terminal rewards in a sequential way, we use the Feynman--Kac formulation. For notational simplicity, we use statewise potential functions $\{g_t(x_t)\}_{t=0}^T$ and define the target path distribution as
\begin{align}
\label{equ:fk_target}
    \pi(dx_{0:T}) = Z^{-1}\gamma(dx_{0:T}) = Z^{-1} P(dx_{0:T}) \prod_{t=0}^T g_t(x_t).
\end{align}
where $\gamma(dx_{0:T})$ is the unnormalized target path measure, and $Z = \mathbb{E}_{X \sim P}[\prod_{t=0}^T g_t(X_t)]$ is the normalizing constant. The terminal-reward alignment target in \eqref{equ:alignment_target} is a special case of \eqref{equ:fk_target}, if 
\begin{align*}
    g_t(x_t) = 1, \quad t=0,\dots,T-1, \qquad
    g_T(x_T) = \exp({r(x_T)}/{\alpha}).
\end{align*}
For a proposal measure $Q$ such that $\gamma \ll Q$, we use the unnormalized importance weight
\begin{align}
\label{equ:importance_weight_def}
    w(X_{0:T}) = \frac{d\gamma}{dQ}(X_{0:T}), \qquad
    \frac{d\pi}{dQ}(X_{0:T}) = \frac{1}{Z} w(X_{0:T}).
\end{align}
The importance weight satisfies $\mathbb{E}_{X\sim Q}[w(X_{0:T})]=Z$. 
Standard SMC approximates $\pi$ with a system of $K$ particles that are propagated, weighted, and resampled over time (refer to \Cref{app:further_details_methodology} for further details). In the standard setting, the proposal is the base model itself, i.e., $Q=P$, and
\begin{align}
\label{equ:smc_global_weight}
     w(X_{0:T}) = \frac{d\gamma}{dP}(X_{0:T}) = \prod_{t=0}^T g_t(X_t).
\end{align}
Recent SMC-based steering methods for diffusion models define intermediate potentials based on reward estimates and apply particle resampling during generation \citep{singhal2025general, wu2023practical}. These methods are appealing because they are training-free and capable of handling terminal or black-box rewards.

However, traditional SMC-based steering propagates particles predominantly according to the base model's dynamics. When the reward-tilted target distribution is substantially distant from the base model's distribution, many particles are initially sampled from low-reward regions and only corrected after propagation via importance weighting and resampling. In long-horizon or high-dimensional generation tasks, this approach leads to poor particle efficiency, weight degeneracy, and high-variance estimates. This fundamental inefficiency motivates modifying the proposal itself using look-ahead information, the core principle behind Twisted SMC.

\section{Methodology: TRI-TSMC}

\label{sec:methodology}

\subsection{Twisted Sequential Monte Carlo for Diffusion Alignment}

The key limitation of standard SMC is that particle propagation proceeds according to the base transition dynamics before the resulting importance weights are calculated and applied for correction. Twisted Sequential Monte Carlo (TSMC) addresses this by modifying the proposal distribution itself. Given the Feynman-Kac target measure $\pi(dx_{0:T}) = Z^{-1}\gamma(dx_{0:T}) \propto P(dx_{0:T})\prod_{t=0}^T g_t(x_t)$, TSMC introduces a sequence of positive twisting functions $\psi=\{\psi_t(\cdot)\}_{t=0}^T$. Intuitively, $\psi_t(x_t)$ provides look-ahead information regarding the likelihood of the current state $x_t$ leading to large future potentials.

Given the base process $P(dx_{0:T})=\mu(dx_0)\prod_{t=1}^T f_t(dx_t\mid x_{t-1})$, the twisted initial distribution and twisted transition kernels are
\begin{align}
\label{equ:twisted_kernel}
    \mu^\psi(dx_0) = \frac{\psi_0(x_0)}{c_\psi}\mu(dx_0), \qquad
    f_t^\psi(dx_t \mid x_{t-1}) = \frac{\psi_t(x_t)}{\tilde{\psi}_{t-1}(x_{t-1})} f_t(dx_t \mid x_{t-1}),
\end{align}
where $c_\psi = \int \psi_0(x_0)\mu(dx_0)$, $\tilde{\psi}_{t-1}(x_{t-1}) = \int \psi_t(x_t)f_t(dx_t\mid x_{t-1})$, $t=1,\dots,T$, 
and we set $\tilde{\psi}_T(\cdot)\equiv 1$. Let $P^\psi$ denote the induced twisted path measure. The objective of twist learning is to choose $\psi$ so that $P^\psi$ becomes a better proposal for the Feynman-Kac target: particles should be guided toward trajectories with large future potentials before being corrected by importance weights. For any given $\psi$, the remaining mismatch between $P^\psi$ and the unnormalized target measure $\gamma$ is defined by the residual importance weight
\begin{align}
\label{equ:twisted_weight}
    w^\psi(X_{0:T}) = \frac{d\gamma}{dP^\psi}(X_{0:T}) = c_\psi \prod_{t=0}^T g_t(X_t)\frac{\tilde{\psi}_t(X_t)}{\psi_t(X_t)}.
\end{align}
Thus, a well-chosen twisting function reduces the variance of this residual importance weight, thereby improving particle efficiency and bringing the particle approximation closer to the target distribution. We provide a detailed decomposition of this incremental weight in Appendix~\ref{app:further_details_methodology}.

Theoretically, \citet{guarniero2017iterated} showed that an optimal twisting function, $\psi^*$, exists that renders the residual importance weight constant, achieving a zero-variance importance sampler. This optimal function satisfies the backward recursion:
\begin{align}
\label{equ:optimal_psi}
    \psi_t^*(x_t) = g_t(x_t)\int \psi_{t+1}^*(x_{t+1}) f_{t+1}(dx_{t+1}\mid x_t), \qquad t=0,\dots,T-1,
\end{align}
with terminal condition $\psi_T^*(x_T)=g_T(x_T)$. \eqref{equ:optimal_psi} immediately implies $\psi_t^*(x_t)=g_t(x_t)\tilde{\psi}_t^*(x_t)$ by definition. In particular, in the terminal-reward case, we have $\psi_t^*(x_t) = \mathbb{E}_{P}[\exp({r(X_T)}/{\alpha})| X_t=x_t]$.
The optimal twist provides a look-ahead value of the current denoising state, and the twisted transition favors next states with larger expected future potential. If $\psi=\psi^*$, the residual importance weight becomes constant and the twisted proposal exactly matches the target path distribution. Since computing $\psi^*$ exactly is intractable in high-dimensional generation, the central algorithmic question is how to learn a parameterized twisting function $\psi(\theta)$ stably from finite samples.

\subsection{Algorithm Design}
Existing SMC-based approaches, such as iAPF (which approximates the optimal twisting functions via recursive regression on simulated particles \citep{guarniero2017iterated}) and TPPF (which directly optimizes neural twisting functions through a path-space KL objective \citep{lu2024guidancetwistedparticlefilter}), have demonstrated effectiveness in simpler sequential statistical inference settings. However, they exhibit instability when applied to the high-dimensional, non-linear dynamics characteristic of generative model alignment. Specifically, direct path-space optimization often suffers from high variance when the current proposal is far from the target distribution, complicating twisting-function learning during early training stages.

To address this, we propose \textbf{Trust-Region Iterative Twisted Sequential Monte Carlo (TRI-TSMC)} in Algorithm~\ref{alg:TR_TSMC} with a flow chart in \Cref{fig:tsmc_flow_chart}. The key idea is to decouple twist learning into two  stages. First, we compute an intermediate path distribution in probability space, constrained to remain within a KL trust region around the current twisted path measure. This step possesses a closed-form solution derived through tempered reweighting of the current proposal. Second, we project this intermediate distribution back onto the parameterized twisted function family by a weighted maximum likelihood. This converts the complex, ideal path-space update into a stable, trainable learning objective for the twisting function.

\begin{figure}[t]
    \centering
    \includegraphics[width=1.0\linewidth]{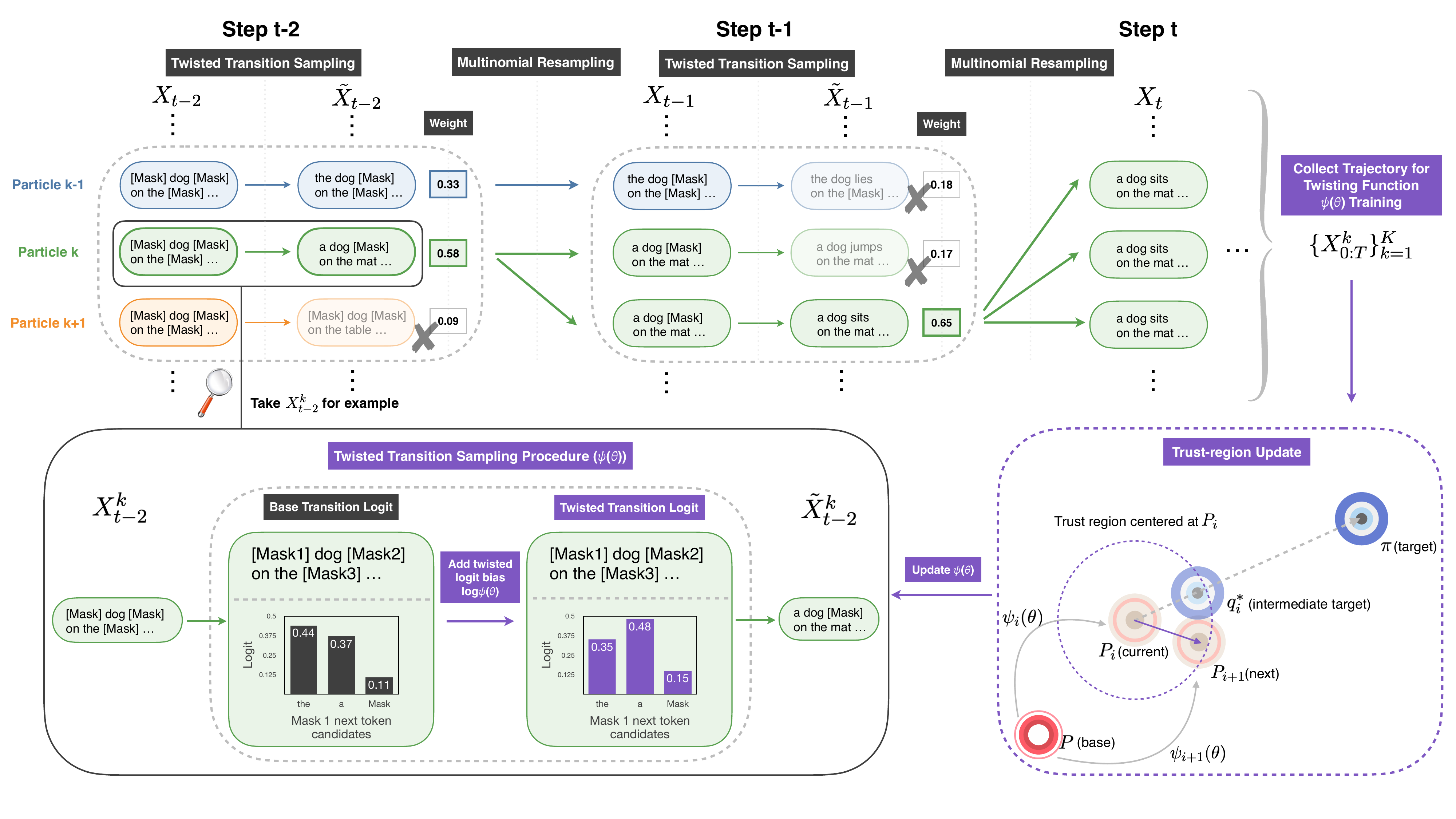}
    \caption{Overview of TRI-TSMC. The sampler alternates between twisted transition sampling and multinomial resampling to propagate particles toward high-reward trajectories. The twisting function modifies the base transition by adding a learned logit bias, providing look-ahead guidance during generation. Collected trajectories are then used in a trust-region update, which computes tempered trajectory weights and trains the next twisting function by weighted maximum likelihood.}
    \label{fig:tsmc_flow_chart}
\end{figure}

\paragraph{Trust Region Update in Path Space.} 

Let $\psi(\theta_i) = \{\psi_t(\cdot;\theta_i)\}_{t=0}^T$ denote the parameterized twisting functions at outer iteration $i$, and let $P_i = P^{\psi(\theta_i)}$ be the corresponding twisted path measure induced by \eqref{equ:twisted_kernel}. We define the normalized density ratio as follows,
\begin{align}
\label{equ:normalized_density_ratio}
    R_i(X_{0:T}) := \frac{d\pi}{dP_i}(X_{0:T}) = \frac{1}{Z} w^{\psi(\theta_i)}(X_{0:T}),
\end{align}
where $w^{\psi(\theta_i)}$ is the unnormalized importance weight in \eqref{equ:twisted_weight}. Our goal is to find an intermediate distribution $q$ that is close to the target distribution $\pi$, while remaining within a trust region around the current path measure $P_i$. We formulate this as
\begin{align}
\label{equ:intermediate_distribution_optimization}
    q_i^* = \argmin_{q} D_{\text{KL}}(q\|\pi)
    \quad \text{s.t.} \quad D_{\text{KL}}\big(q \| P_i\big) \leq \epsilon,
\end{align}
where $\epsilon > 0$ is the defined trust-region radius.

The next lemma shows that this trust-region update has a closed-form solution and achieves a monotone improvement in distribution space.

\begin{lemma}[Trust-region path-space update]
\label{lem:trust_region}
Let $P_i = P^{\psi(\theta_i)}$ be the current twisted path measure, and let $q_i^*$ be the solution to \eqref{equ:intermediate_distribution_optimization}. Then there exists $\lambda_i\geq 0$ such that
\begin{align}
\label{equ:trust_region_closed_form}
    \frac{dq_i^*}{dP_i}(X_{0:T})
    \propto (R_i(X_{0:T}))^{\frac{1}{1+\lambda_i}}
    \propto \big(w^{\psi(\theta_i)}(X_{0:T})\big)^{\frac{1}{1+\lambda_i}}.
\end{align}
Moreover, we have $D_{\text{KL}}(q_i^*\|\pi)\leq D_{\text{KL}}(P_i\|\pi)$.
\end{lemma}

\paragraph{Efficient Implementation via Weighted Maximum Likelihood.} 

In practice, we solve the dual problem using sampled trajectories $\{\xi^k\}_{k=1}^K\sim P_i$. Since $R_i=Z^{-1}w^{\psi(\theta_i)}$, the empirical dual problem simplifies to
\begin{align}
\label{equ:empirical_dual}
    \textstyle \hat{\lambda}_i = \argmin_{\lambda \geq 0} \Big[\lambda \epsilon + (1+\lambda) \log\Big(\frac{1}{K} \sum_{k=1}^K \big(w^{\psi(\theta_i)}(\xi^k)\big)^{\frac{1}{1+\lambda}}\Big)\Big].
\end{align}
Therefore, based on \eqref{equ:trust_region_closed_form}, the tempered regression weights are computed as $\hat{w}_k = \frac{(w^{\psi(\theta_i)}(\xi^k))^{\hat{\tau}_i}}{\sum_{j=1}^K (w^{\psi(\theta_i)}(\xi^j))^{\hat{\tau}_i}}$,
where $\hat{\tau}_i = 1/(1+\hat{\lambda}_i)$. We then fit the twisting function by weighted maximum likelihood:
\begin{align}
\label{equ:parameter_update_loss}
   \mathcal{L}(\theta) = -\sum_{k=1}^K \hat{w}_k \log P^{\psi(\theta)}(\xi^k).
\end{align}
This objective \eqref{equ:parameter_update_loss} is the empirical projection of the intermediate trust-region target back to the parameterized twisted family. A more detailed derivation is given in \Cref{app:further_details_methodology}.

\begin{algorithm}[t]
\caption{Trust-Region Iterative Twisted Sequential Monte Carlo (TRI-TSMC)}
\label{alg:TR_TSMC}
\begin{algorithmic}[1]
    \REQUIRE Base model $\{\mu, f_t\}_{t=1}^T$, potentials $\{g_t\}_{t=0}^T$, trust-region size $\epsilon$, number of particles $K$, outer iterations $I$.
    
    \STATE Initialize twisting parameters $\theta_0$.
    
    \FOR{$i = 0, 1, \dots, I-1$}

        \STATE Let $P_i = P^{\psi(\theta_i)}$ be the current twisted path measure. \label{line:define_path_measure}
        
        \STATE Run the twisted SMC particle system under $\psi(\theta_i)$ to obtain ancestral trajectories $\{X_{0:T}^{(k)}\}_{k=1}^K$. \label{line:tsmc_sampling}

        \STATE Compute residual trajectory weights $w^{(k)} = w^{\psi(\theta_i)}(X_{0:T}^{(k)})$ using \eqref{equ:twisted_weight}. \label{line:compute_weight}
        
        \STATE Solve the empirical dual problem \eqref{equ:empirical_dual} to obtain $\hat{\lambda}_i$ and set $\hat{\tau}_i = 1/(1+\hat{\lambda}_i)$. \label{line:dual_solution}
        
        \STATE Calculate normalized tempered weights $\hat{w}_k = \frac{(w^{(k)})^{\hat{\tau}_i}}{\sum_{j=1}^K (w^{(j)})^{\hat{\tau}_i}}$,  for $k=1,\dots,K$. \label{line:compute_tempered_weight}

        \STATE Update $\theta_{i+1}$ by optimizing from $\theta_i$ on $\mathcal{L}_i(\theta) = - \sum_{k=1}^K \hat{w}_k \log P^{\psi(\theta)}(X_{0:T}^{(k)})$. \label{line:optimization}

    \ENDFOR
    
    \STATE \textbf{return} $\theta_I$.
\end{algorithmic}
\end{algorithm}

\paragraph{Algorithm Interpretation.} 

At a high level, \Cref{alg:TR_TSMC} alternates between a twisted SMC sampling stage and a trust-region learning stage. At outer iteration $i$, Line~\ref{line:define_path_measure} defines the current twisted path measure $P_i=P^{\psi(\theta_i)}$ induced by the current twisting function. The first stage (Line~\ref{line:tsmc_sampling}-\ref{line:compute_weight}) generates particles (trajectories) using the current twisted SMC sampler. Specifically, particles are propagated by the twisted kernels in \eqref{equ:twisted_kernel} and resampled using the residual incremental weights. At each step $t$,
\begin{align*}
    W_t^k = \frac{w_t^{\psi(\theta_i)}(X_t^k)}{\sum_{\ell=1}^K w_t^{\psi(\theta_i)}(X_t^\ell)}, \qquad A_t^k\sim \mathrm{Cat}(W_t^{1:K}), \qquad X_{t+1}^k\sim f_{t+1}^{\psi(\theta_i)}(\cdot\mid X_t^{A_t^k}).
\end{align*}
The resulting resampled trajectories $\{X_{0:T}^{(k)}\}_{k=1}^K$ form a finite-particle approximation of the current twisted proposal. Line~\ref{line:compute_weight} then computes the residual trajectory weights $w^{(k)}$, which measure the remaining mismatch between the current twisted proposal and the unnormalized target measure. The second stage (Lines~\ref{line:dual_solution}-\ref{line:optimization}) performs the trust-region update and projection. Line~\ref{line:dual_solution} solves the empirical dual problem to obtain $\hat{\lambda}_i$ and the trust-region temperature $\hat{\tau}_i=1/(1+\hat{\lambda}_i)$. Line~\ref{line:compute_tempered_weight} uses this temperature to construct tempered trajectory weights. These weights define the empirical intermediate path distribution associated with the exact trust-region update. Finally, Line~\ref{line:optimization} fits the next twisting function by weighted maximum likelihood, so that $P^{\psi(\theta_{i+1})}$ approximates this intermediate target within the parameterized twisted family.

\section{Theoretical Analysis}
\label{sec:theoretical_analysis}

This section analyzes the theoretical guarantees of TRI-TSMC. We first formalize the value-function interpretation of the optimal twisting function and show that the optimal twisting function yields a zero-variance sampler. We demonstrate that the exact KL-constrained update follows an escort path from the current twisted proposal toward the target distribution. Finally, we prove that the weighted maximum-likelihood step is a forward-KL projection and that the trust-region path reduces residual importance-weight variance. Throughout this section, we assume that $g_t(x_t)>0$ for all $t$ and all $x_t$. We denote $P_i=P^{\psi(\theta_i)}$ and use $R_i$ defined in \eqref{equ:normalized_density_ratio} for path-space analysis. Complete proofs in this section are given in \Cref{sec:appendix_proof}.

\begin{theorem}[Equivalence to soft value function]
\label{thm:soft_value_equivalence}
Let $\psi^*=\{\psi_t^*\}_{t=0}^T$ be the optimal twisting function sequence satisfying \eqref{equ:optimal_psi}. Define that
\begin{align}
\label{equ:value_function_def}
    V_t^*(x_t) := \alpha \log \psi_t^*(x_t), \qquad t=0,\dots,T.
\end{align}
Then, for $t=0,\dots,T-1$, $V_t^*$ satisfies the soft Bellman recursion
\begin{align}
\label{equ:soft_bellman_general}
    V_t^*(x_t) = \alpha \log g_t(x_t) + \alpha \log \int \exp\big({V_{t+1}^*(x_{t+1})}/{\alpha}\big) f_{t+1}(dx_{t+1}\mid x_t),
\end{align}
with terminal condition $V_T^*(x_T) = \alpha \log g_T(x_T)$. Crucially, under the optimal twist, $w^{\psi^*}(X_{0:T}) = Z$, which implies a zero-variance importance sampler.
\end{theorem}

\begin{remark}
In the terminal-reward case, where $g_t(x_t)\equiv 1$ for $t<T$ and $g_T(x_T)=\exp(r(x_T)/\alpha)$, \Cref{thm:soft_value_equivalence} reduces to $V_t^*(x_t) = \alpha \log \int \exp({V_{t+1}^*(x_{t+1})}/{\alpha}) f_{t+1}(dx_{t+1}\mid x_t)$ for $t=0,\dots,T-1$, 
with $V_T^*(x_T)=r(x_T)$. Thus, in diffusion alignment, the optimal twisting function can be viewed as the soft value function of the current denoising state.
\end{remark}

We next analyze the exact path-space update used to move the current proposal toward the target. \Cref{lem:trust_region} gives the closed-form trust-region update. The following proposition shows that the exact update lies on an escort path between the current twisted proposal $P_i$ and the target distribution $\pi$.

\begin{proposition}
\label{prop:escort_geometry}
Suppose that, for every $\tau \in [0,1]$, $\mathbb{E}_{X\sim P_i}[R_i(X)^\tau |\log R_i(X)|^2] < \infty$. Define
\begin{alignat}{2}
    q_{i,\tau}(dX) &:= \frac{R_i(X)^\tau}{\mathbb{E}_{X\sim P_i}[R_i(X)^\tau]} P_i(dX),\label{equ:escort_family}\\
    \Lambda_i(\tau) &:= \log \mathbb{E}_{X\sim P_i}[R_i(X)^\tau].
\label{equ:log_partition_escort}
\end{alignat}
Then $q_{i,0}=P_i$, $q_{i,1}=\pi$, and the exact trust-region solution in \eqref{equ:intermediate_distribution_optimization} satisfies $q_i^* = q_{i,\tau_i}$ and $\tau_i = 1/{(1+\lambda_i)}$,
where $\lambda_i$ is the dual optimizer in \Cref{lem:trust_region}. Moreover, $\tau \mapsto D_{\text{KL}}(q_{i,\tau}\|P_i)$ is nondecreasing on $[0,1]$, while $\tau \mapsto D_{\text{KL}}(q_{i,\tau}\|\pi)$ is nonincreasing on $[0,1]$. If $P_i \neq \pi$, then these monotonicities are strict on $(0,1]$ and $[0,1)$, respectively. Consequently, for every $\epsilon \in (0, D_{\text{KL}}(\pi\|P_i))$, there exists a unique $\tau_i \in (0,1)$ such that $D_{\text{KL}}(q_{i,\tau_i}\|P_i)=\epsilon$.
\end{proposition}

The escort update $q_i^*$ is defined in the space of path measures and is generally not represented exactly by the parameterized twisted family. Therefore, the implementable algorithm must project this intermediate target back to a tractable family $P^{\psi(\theta)}$. The next lemma shows that the weighted maximum-likelihood step used by TRI-TSMC is precisely the forward-KL projection of $q_i^*$.

\begin{lemma}
\label{lem:projection_as_mle}
Define the projection objective $L_i(\theta):= -\mathbb{E}_{X\sim q_i^*}[\log P^{\psi(\theta)}(X)]$. Then minimizing $L_i(\theta)$ over $\theta$ is equivalent to minimizing $D_{\text{KL}}(q_i^* \| P^{\psi(\theta)})$. More precisely,
\begin{align*}
    L_i(\theta) = D_{\text{KL}}\big(q_i^* \| P^{\psi(\theta)}\big) - \mathbb{E}_{X\sim q_i^*}\big[\log q_i^*(X)\big].
\end{align*}
Consequently, if $\theta_{i+1}$ satisfies $L_i(\theta_{i+1}) \leq \inf_{\theta} L_i(\theta) + \eta_i$, then, with $P_{i+1}:=P^{\psi(\theta_{i+1})}$, we have
\begin{align*}
    D_{\text{KL}}\big(q_i^* \| P_{i+1}\big) \leq \inf_{\theta} D_{\text{KL}}\big(q_i^* \| P^{\psi(\theta)}\big) + \eta_i.
\end{align*}
\end{lemma}

Thus, the weighted maximum-likelihood objective in \eqref{equ:parameter_update_loss} is the empirical version of a forward-KL projection from the exact trust-region target to the learned twisted path measure. \Cref{sec:appendix_proof} further gives a projection error bound for this step. Finally, we analyze the stability of the tempered reweighting step. Since the projection objective is weighted by importance weights, a natural measure of stability is the variance of the residual weights after moving along the escort path. 

\begin{theorem}[Variance reduction]
\label{thm:variance_reduction}
Assume that $R_i \in L^2(P_i)$. Let $q_{i,\tau}$ and $\Lambda_i(\tau)$ be defined by \eqref{equ:escort_family} and \eqref{equ:log_partition_escort}, and define $\chi^2(\pi\|q) := \int (\frac{d\pi}{dq}(X)-1)^2 q(dX)$. Then, for every $\tau\in[0,1]$,
\begin{align}
\label{equ:variance_gamma_ratio}
    Z^{-2}{\mathrm{Var}_{X\sim q_{i,\tau}}\bigg[\frac{d\gamma}{dq_{i,\tau}}(X)\bigg]}= \chi^2(\pi\|q_{i,\tau}).
\end{align}
Further, $\tau \mapsto \chi^2(\pi\|q_{i,\tau})$ is nonincreasing on $[0,1]$ and strictly decreasing on $[0,1)$ whenever $P_i\neq \pi$.
\end{theorem}

\begin{remark}
\Cref{thm:variance_reduction} gives a distribution-level explanation for the stabilizing effect of the trust-region update. Along the exact escort path, it monotonically reduces the normalized variance of the residual importance weights. This suggests that, as the proposal moves toward the target, fewer trajectories dominate the weighted maximum-likelihood projection, leading to a more stable update.
\end{remark}

\section{Experiments}
\label{sec:experiments}

In this section, we evaluate TRI-TSMC on two inference-time alignment tasks including discrete diffusion text generation and text-to-image diffusion generation. The first task studies whether TRI-TSMC can steer a Masked Diffusion Language Model toward fluent and grammatically acceptable text. The second task evaluates TRI-TSMC on text-to-image Stable Diffusion models toward images with higher ImageReward. We present the improved performance of our proposed TRI-TSMC across different tasks compared with Best-of-$N$ and FK-Steering under comparable sampling budgets.

\subsection{Discrete Diffusion Text Generation}
\label{sec:disc-text}

\paragraph{Experimental Setup.} 

We first evaluate TRI-TSMC on inference-time alignment for discrete diffusion text generation. We use the Masked Diffusion Language Model (MDLM) \citep{sahoo2024simple} trained on OpenWebText as the base generator. Our evaluation uses 15 discriminative prompts from the PPLM benchmark \citep{dathathri2019plug}. For each prompt and each random seed in $\{42,43,44\}$, we generate 20 batches, yielding 300 samples per seed and 900 samples in total. All experiments use 200 diffusion steps, and particle resampling is performed every 20 steps.

For text generation experiments, we use a mixed steering reward that combines a PPL reward with a grammar reward CoLA:
$r_{\mathrm{mix}} = r_{\mathrm{PPL}} + \beta r_{\mathrm{CoLA}}.$
Here $r_{\mathrm{PPL}}$ is the log negative perplexity score derived from a language-model perplexity evaluator,
$r_{\mathrm{CoLA}}$ is the acceptability score from a CoLA classifier
\citep{warstadt2019neural}, and $\beta$ is a fixed mixing weight. 
In the main experiment, we use GPT-2 \citep{radford2019language} as reward function and provide Qwen2.5-1.5B \citep{hui2024qwen2} based results in \Cref{app:qwen_evaluation}. We report perplexity (PPL) under the corresponding fluency evaluator, CoLA score, and Dist-1/2/3 \citep{li2016diversity}.

\paragraph{Baselines.} 

We compare TRI-TSMC against three baselines: (i) \textbf{Base}, which samples directly from the pretrained MDLM without steering ($K=1$); (ii) \textbf{Best-of-N (BoN)}, which draws $K \in \{16,48\}$ independent samples and returns the one with the highest final reward; (iii) \textbf{FK-Steering}, the particle-based inference-time steering method of \citet{singhal2025general} using the reward $r_{\mathrm{PPL}}$ to define potential functions for alignment, with $K \in \{16,48\}$.
For BoN and FK-Steering, we report $K\in\{16,48\}$. The $K=16$ setting matches
the number of particles used by TRI-TSMC in each sampling iteration, while $K=48$
matches the total number of sampled trajectories used by TRI-TSMC across all iterations.

\paragraph{Implementation Details.}

We use the same resampling interval for FK-Steering and TRI-TSMC, performing particle resampling every 20 reverse diffusion steps. For GPT-2-based alignment, TRI-TSMC uses the \texttt{max} potential to construct local Feynman--Kac weights from intermediate reward estimation. The twisting function is parameterized by an MLP. Full hyperparameters, reward weights, potential definitions, and additional implementation details are provided in \Cref{app:implement_details}.

\paragraph{Results Analysis.} 

\Cref{fig:gpt2_overall_result} reports GPT-2-based PPL reward alignment results on MDLM. TRI-TSMC achieves the lowest perplexity among all inference-time baselines, reducing PPL from $36.702$ under the strongest FK-Steering baseline ($K=48$) and $44.109$ under BoN ($K=48$) to $29.781$ with $K=16$ particles. TRI-TSMC also attains the highest CoLA score, improving over FK-Steering ($K=48$) from $0.393$ to $0.727$. In terms of diversity, TRI-TSMC has lower Dist-1/2/3 scores than FK-Steering, decreasing from $0.520/0.831/0.888$ under FK-Steering ($K=48$) to $0.488/0.774/0.833$. The Base model still has the highest Dist-$n$ values overall, which is expected because it performs no reward alignment and therefore samples more broadly. Overall, these results show that TRI-TSMC gives the strongest improvement on the primary alignment objective and substantially improves CoLA, while trading off some diversity. Note that TRI-TSMC uses $K=16$ particles for each outer iteration and runs $3$ outer iterations, giving the same total trajectory budget as FK-Steering with $K=48$ while using a smaller particle set per iteration.

We further provide Qwen2.5-1.5B-based evaluation in \Cref{app:qwen_evaluation}. Similarly, TRI-TSMC achieves the best aligned metric PPL and the best CoLA score. In addition, we conduct an FK-Steering mixed-reward ablation study to further study the gain of mixed reward design in \Cref{app:fk_mixed}. Finally, we further provide a qualitative analysis of generated text samples that illustrate the same trade-off between fluency, acceptability, and diversity in \Cref{app:text_samples}.

\begin{figure}[t]
\centering
\begin{minipage}[t]{0.54\textwidth}
\vspace{12mm}
\centering
\setlength{\tabcolsep}{2.5pt}
\fontsize{8.5pt}{11pt}\selectfont
\begin{NiceTabular}{lcccccc}
\CodeBefore
  \rowcolor{gray!10}{7}
  \columncolor[opacity=0.2]{blue}{3}
\Body
\toprule
Method & $K$  & PPL(*) $\downarrow$ & CoLA $\uparrow$ & Dist-1 $\uparrow$ & Dist-2 $\uparrow$ & Dist-3 $\uparrow$ \\
\midrule
Base & 1  & 104.427 & 0.223 & \textbf{0.594} & \textbf{0.919} & \textbf{0.941} \\
BoN & 16  & 51.301  & 0.347 & 0.552 & 0.879 & 0.922 \\
BoN & 48  & 44.109  & 0.371 & 0.532 & 0.852 & 0.904 \\
FK-Steering & 16 & 43.050  & 0.390 & 0.528 & 0.849 & 0.905 \\
FK-Steering & 48 & 36.702  & 0.393 & 0.520 & 0.831 & 0.888 \\
\textbf{TRI-TSMC} & 16 & \textbf{29.781} & \textbf{0.727} & 0.488 & 0.774 & 0.833 \\
\bottomrule
\end{NiceTabular}
\end{minipage}
\hfill
\begin{minipage}[t]{0.42\textwidth}
\vspace{0pt}
\centering
\includegraphics[width=\linewidth]{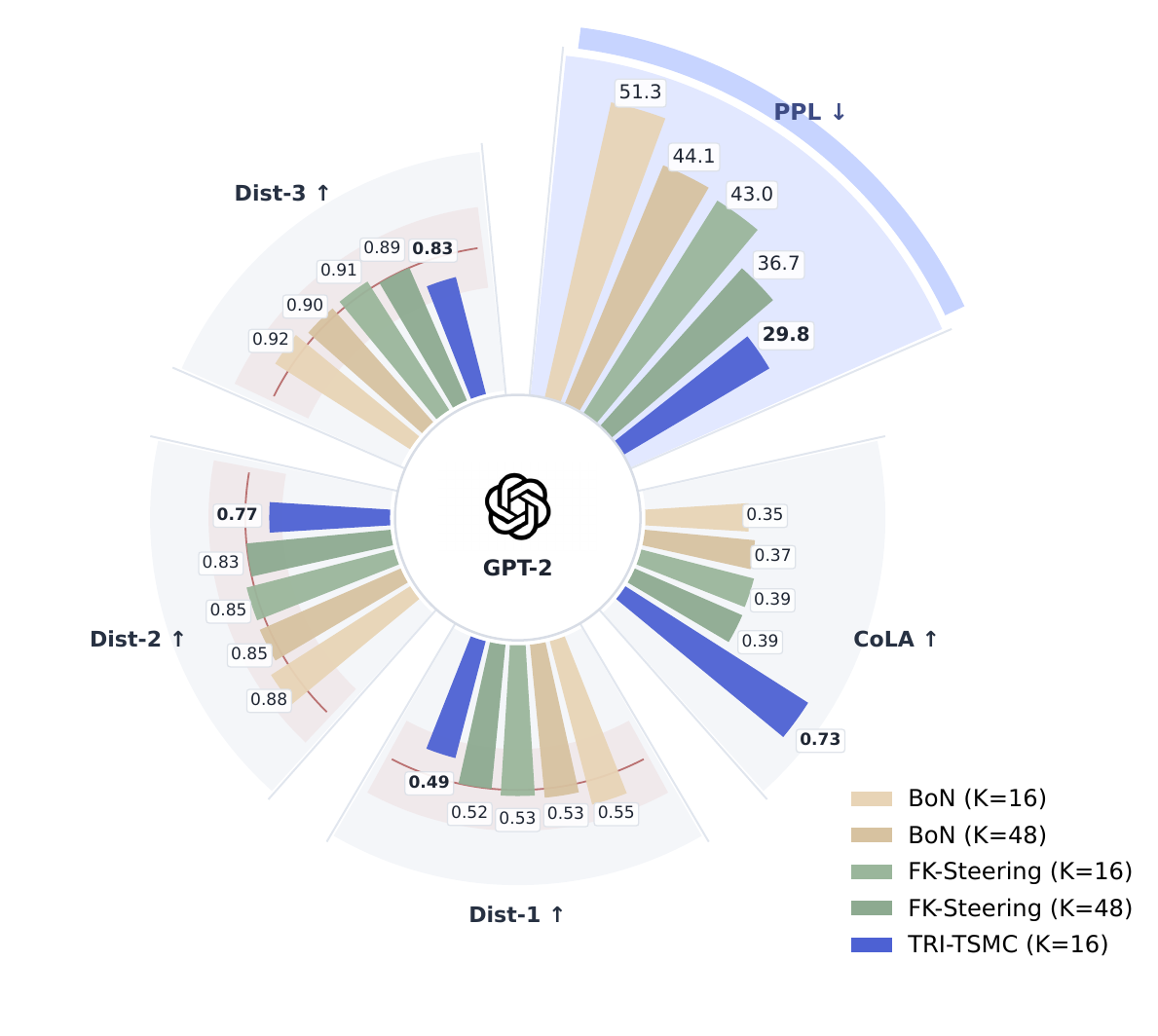}
\end{minipage}
\caption{
Alignment results on MDLM under GPT-2-based evaluation. 
\textbf{LHS:} Quantitative results for Base, BoN, FK-Steering, and TRI-TSMC; PPL is the primary alignment objective, while CoLA and Dist-$n$ are secondary quality indicators. 
\textbf{RHS:} Polar bar visualization of the same results, with the PPL sector visually emphasized. The shaded bands in the secondary-metric sectors (Dist-1/2/3) indicate the indicate a $\pm$15\% tolerance band centered at FK-Steering ($K=48$).
}
\label{fig:gpt2_overall_result}
\end{figure}

\subsection{Text-to-Image Diffusion Models}

\paragraph{Experimental setup.} 

We further evaluate TRI-TSMC on text-to-image inference-time alignment. We leverage the Stable Diffusion \citep{rombach2022high,podell2023sdxl} as the base model and sample images conditioned on a text prompt. We use Stable Diffusion v1.5 (SD v1.5) as the main base generator and report additional results on SDXL to show the generalizability of TRI-TSMC. The reward used for steering is the off-the-shelf ImageReward (IR) model \citep{xu2023imagereward}, which serves as the primary alignment objective in this experiment.  We also report GenEval \citep{ghosh2023geneval} for prompt fidelity and HPS \citep{wu2023human} as an additional human preference metric. For evaluating the IR and HPS score, we use the prompts from the ImageReward benchmark, which contains 100 prompts to evaluate alignment performance. We also use the official GenEval benchmark which contains 553 prompts to evaluate compositional prompt fidelity. We choose the same three random seeds as the text generation task with $\{42,43,44\}$ to report the results.

\paragraph{Baselines.}

We compare TRI-TSMC against four baselines: 
(i) \textbf{Base}, which samples directly from the diffusion model without reward steering ($K=1$); 
(ii) \textbf{Best-of-$N$ (BoN)}, which draws $K \in \{4,8\}$ independent images and returns the one with the highest IR score; 
(iii) \textbf{FK-Steering}, the particle-based inference-time steering method of \citet{singhal2025general} that uses IR-based Feynman--Kac potentials for particle weighting and resampling, with $K \in \{4,8\}$; 
(iv) \textbf{Post-training baselines}, including Diffusion-DPO \citep{wallace2024diffusion}. 
For BoN and FK-Steering, we report both $K=4$ and $K=8$ results. 
The $K=4$ setting matches the number of particles used by TRI-TSMC in one iteration, while the $K=8$ setting matches the total number of sampled trajectories used by TRI-TSMC across all iterations.

\paragraph{Implementation Details.} 

For all image-generation experiments, we use the base diffusion model as the proposal generator and adopt DDIM sampling with $100$ denoising steps and classifier-free guidance scale $7.5$, consistent with the setup used by \citet{singhal2025general}.  For particle-based inference-time methods, we use a resampling frequency of $20$, corresponding to the resampling schedule $[20, 40, 60, 80, 99]$.  At each resampling stage, we compute an intermediate reward on the denoised prediction $\hat{x}_0$, which provides a practical estimate of final image quality before the terminal denoising step. For SD v1.5 tasks, TRI-TSMC uses the \texttt{diff} potential to construct local Feynman--Kac weights. The twisting function is parameterized by a CNN. We use $K=4$ particles for all experiments in text-to-image tasks. More details can be found in \Cref{app:image_implement}.

\paragraph{Results analysis.} 

\Cref{tab:text2image_main_max} reports the main text-to-image results on SD v1.5.
IR is the primary objective in this experiment.
TRI-TSMC achieves the best IR score among all baselines.
It improves over Best-of-$N$ ($K=8$) from $0.945$ to $0.979$, and over FK-Steering ($K=8$) from $0.930$ to $0.979$. The improvement is larger when compared with the $K=4$ baselines. Overall, these results show that TRI-TSMC steers the sampling process toward images preferred by IR and achieves the best score among all the baselines with comparable sample trajectories. On the secondary metrics, TRI-TSMC obtains an HPS score of $0.270$. This is comparable to FK-Steering ($K=8$), but lower than Best-of-$N$ ($K=8$). This is reasonable because TRI-TSMC is trained and selected using IR rather than HPS. For GenEval, TRI-TSMC is comparable to FK-Steering in the current evaluation. This suggests that the improvement in IR does not come with a large drop in prompt fidelity.

We also provide additional experiment results on the SDXL in the \Cref{tab:t2i_sdxl} in \Cref{app:SDXL_results}. The results demonstrate that TRI-TSMC can also achieve the best performance on IR and comparable results in GenEval and HPS scores. As a qualitative comparison, \Cref{fig:footage_of_an_astronaut} shows the particle set before and after one trust-region twist update. Compared with the initialization particles, the updated particles are more consistently steered toward generations preferred by IR. This example suggests that TRI-TSMC does not simply change the particles toward high reward but also remains visually plausible and prompt-aligned.

\begin{table}[h]
    \centering
    \caption{Main results on text-to-image alignment with SD v1.5. ImageReward and HPS report the best-of-$K$ score averaged over three random seeds, with standard deviation. GenEval reports correct prompts averaged over three random seeds.}
    \begin{NiceTabular}{l c c c c}
        \CodeBefore
        \rowcolor{gray!10}{8}
        \columncolor[opacity=0.2]{blue}{3}
        \Body
        \toprule
        Method & $K$ & IR(*) $\uparrow$ & GenEval $\uparrow$ & HPS $\uparrow$ \\
        \midrule
        Base & 1
        & $0.233$ & $0.428$ & $0.246$ \\
        
        DPO & 1
        & $0.359$ & $0.441$ & $0.256$ \\
        
        Best-of-$N$ & 4
        & $0.750$ & $0.599$ & $0.267$ \\
        
        Best-of-$N$ & 8
        & $0.945$ & $\mathbf{0.670}$ & $\mathbf{0.276}$ \\
        
        FK-Steering & 4
        & $0.718$ & $0.553$ & $0.261$ \\
        
        FK-Steering & 8
        & $0.930$ & $0.630$ & $0.270$ \\
        
       \textbf{TRI-TSMC} & 4
        & $\mathbf{0.979}$ & $0.601$ & $0.270$ \\
        \bottomrule
    \end{NiceTabular}
    \label{tab:text2image_main_max}
\end{table}

\begin{figure}[h]
    \centering
    \includegraphics[width=\linewidth]{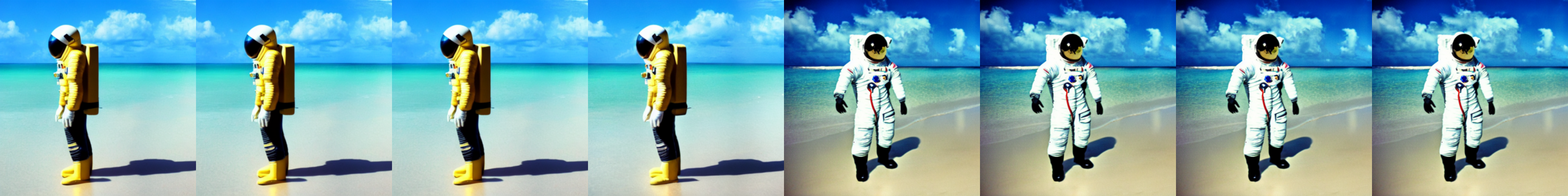}
    \caption{TRI-TSMC qualitative samples for the prompt ``footage of an astronaut.'' The left four images correspond to the initialization particles, while the right four images show the samples after one TRI-TSMC iteration. This visualization shows how TRI-TSMC updates the particle distribution toward higher-reward generations while remaining prompt-aligned.}
    \label{fig:footage_of_an_astronaut}
\end{figure}

\section{Conclusion}
\label{sec:conclusion}

We proposed Trust-Region Iterative Twisted Sequential Monte Carlo (TRI-TSMC) for inference-time alignment of diffusion-based generative models. The method learns look-ahead twisting functions through an exact KL-constrained path-space update followed by weighted maximum-likelihood projection. Theoretically, we formalized the value-function interpretation of the optimal twisting function and showed that the optimal twisting function yields a zero-variance sampler. We also showed that the exact trust-region update follows an escort path toward the target distribution, that the practical projection is a forward-KL projection, and that the trust-region path reduces residual importance-weight variance. Empirically, we evaluated TRI-TSMC on both discrete diffusion text generation and text-to-image diffusion alignment. On MDLM text generation under GPT-2-based evaluation, TRI-TSMC substantially reduced PPL and improved CoLA over baselines. Additional Qwen2.5-1.5B-based evaluation and qualitative text analysis further support the effectiveness of our method. On text-to-image generation with Stable Diffusion, TRI-TSMC improved ImageReward over matched-budget baselines. These results show that trust-region learning of twisting functions improves primary alignment objectives. Future work includes improving the quality--diversity trade-off, scaling to stronger image models, and distilling TRI-TSMC trajectories into fast aligned generators.

\newpage
\appendix

\section{Further Details in Preliminary and Methodology}

\label{app:further_details_methodology}

This section provides additional details for \Cref{sec:preliminary} and \Cref{sec:methodology}.

\paragraph{Standard SMC sampling.}

Consider the Feynman--Kac target
\begin{align*}
    \gamma(dx_{0:T}) = P(dx_{0:T})\prod_{t=0}^T g_t(x_t), \qquad
    P(dx_{0:T})=\mu(dx_0)\prod_{t=1}^T f_t(dx_t\mid x_{t-1}).
\end{align*}
Standard SMC keeps a set of particles and their ancestral paths. We initialize $X_0^k\sim \mu$ for $k=1,\dots,K$. At time $t$, the current particles are weighted by the local potential,
\begin{align*}
    W_t^k = \frac{g_t(X_t^k)}{\sum_{\ell=1}^K g_t(X_t^\ell)}, \qquad k=1,\dots,K.
\end{align*}
We then sample ancestors $A_t^k\sim \mathrm{Cat}(W_t^{1:K})$ and propagate
\begin{align*}
    X_{t+1}^k \sim f_{t+1}(\cdot\mid X_t^{A_t^k}), \qquad k=1,\dots,K.
\end{align*}
The ancestral path is copied together with the ancestor: if $A_t^k=a$, then the new path of particle $k$ first inherits $X_{0:t}^a$ and then appends the newly sampled state $X_{t+1}^k$. Thus, high-weight particles are duplicated and low-weight particles are removed before the next propagation step. In practice, resampling does not need to occur at every step; if $t$ is not a resampling stage, we simply set $A_t^k=k$ and propagate each particle forward without resampling.

This procedure approximates the Feynman--Kac path measure by repeatedly correcting the base dynamics with the potentials $g_t$. However, if the base dynamics $f_t$ is far from the reward-tilted target, many particles may still move through low-reward regions before being corrected by weighting and resampling. Twisted SMC modifies this propagation step directly.

\paragraph{Twisted SMC sampling.}

Given a positive twisting function sequence $\psi=\{\psi_t\}_{t=0}^T$, twisted SMC uses the twisted initial distribution and transition kernels
\begin{align*}
    \mu^\psi(dx_0) = \frac{\psi_0(x_0)}{c_\psi}\mu(dx_0), \qquad
    f_t^\psi(dx_t\mid x_{t-1}) = \frac{\psi_t(x_t)}{\tilde{\psi}_{t-1}(x_{t-1})}f_t(dx_t\mid x_{t-1}),
\end{align*}
where $c_\psi=\int \psi_0(x_0)\mu(dx_0)$ and $\tilde{\psi}_{t-1}(x_{t-1})=\int \psi_t(x_t)f_t(dx_t\mid x_{t-1})$. The role of $\psi_t$ is to give look-ahead guidance: states with larger estimated future potential are assigned larger probability under the twisted transition.

After twisting the transition, the original potentials are replaced by residual incremental weights. With the convention $\tilde{\psi}_T(\cdot)\equiv 1$, these residual weights are
\begin{align*}
    w_0^\psi(x_0) = g_0(x_0)c_\psi\frac{\tilde{\psi}_0(x_0)}{\psi_0(x_0)}, \qquad
    w_t^\psi(x_t) = g_t(x_t)\frac{\tilde{\psi}_t(x_t)}{\psi_t(x_t)}, \quad t=1,\dots,T.
\end{align*}
Their product recovers the global residual importance weight
\begin{align*}
    w^\psi(X_{0:T}) = w_0^\psi(X_0)\prod_{t=1}^T w_t^\psi(X_t) = \frac{d\gamma}{dP^\psi}(X_{0:T}).
\end{align*}
Therefore, twisting changes the proposal but preserves the target through the residual weights.

The twisted SMC sampler is the same resample--propagate procedure as standard SMC, except that both the propagation kernel and the weights are changed. We initialize $X_0^k\sim \mu^\psi$. At a resampling stage $t$, we compute
\begin{align*}
    W_t^k = \frac{w_t^\psi(X_t^k)}{\sum_{\ell=1}^K w_t^\psi(X_t^\ell)}, \qquad k=1,\dots,K.
\end{align*}
We then sample ancestors $A_t^k\sim \mathrm{Cat}(W_t^{1:K})$ and propagate with the twisted kernel,
\begin{align*}
    X_{t+1}^k \sim f_{t+1}^\psi(\cdot\mid X_t^{A_t^k}), \qquad k=1,\dots,K.
\end{align*}
As in standard SMC, the ancestral path of particle $k$ is updated by copying the path of its sampled ancestor and appending the new state. If no resampling is performed at step $t$, we set $A_t^k=k$ and only apply the twisted transition.

This makes the role of twisting functions explicit. Standard SMC uses the base kernel $f_t$ and relies on the potentials $g_t$ to correct particles after propagation. Twisted SMC instead uses $\psi_t$ to modify the transition before propagation and then uses the residual weights $w_t^\psi$ to correct the remaining mismatch. A good twisting function therefore reduces the variance of the residual weights. In the ideal case, the optimal twisting function makes these residual weights constant, giving a zero-variance importance sampler.

\paragraph{Incremental weights for twisted SMC.}

Using \eqref{equ:twisted_weight}, the residual importance weight under the twisted proposal can be decomposed as
\begin{align*}
    w^\psi(X_{0:T}) = w_0^\psi(X_0)\prod_{t=1}^T w_t^\psi(X_t),
\end{align*}
where
\begin{align*}
    w_0^\psi(x_0) = g_0(x_0)c_\psi\frac{\tilde{\psi}_0(x_0)}{\psi_0(x_0)}, \qquad w_t^\psi(x_t) = g_t(x_t)\frac{\tilde{\psi}_t(x_t)}{\psi_t(x_t)}, \quad t=1,\dots,T.
\end{align*}
Here we use the convention $\tilde{\psi}_T(\cdot)\equiv 1$. If $\psi=\psi^*$, then $\psi_t^*(x_t)=g_t(x_t)\tilde{\psi}_t^*(x_t)$ for $t<T$ and $\psi_T^*(x_T)=g_T(x_T)$. Therefore $w_t^{\psi^*}(x_t)\equiv 1$ for $t=1,\dots,T$. For the initial term, $w_0^{\psi^*}(x_0)=c_{\psi^*}$, and applying the recursion in \eqref{equ:optimal_psi} gives
\begin{align*}
    c_{\psi^*} = \int \mu(dx_0)\psi_0^*(x_0) = \int P(dx_{0:T})\prod_{t=0}^T g_t(x_t) = Z.
\end{align*}
Thus $w^{\psi^*}(X_{0:T})=Z$, which gives the zero-variance importance sampler.

\paragraph{Geometric annealing interpretation.}

The trust-region update in \Cref{lem:trust_region} gives
\begin{align*}
    \frac{dq_i^*}{dP_i}(X_{0:T}) \propto \bigg(\frac{d\pi}{dP_i}(X_{0:T})\bigg)^{\frac{1}{1+\lambda_i}}.
\end{align*}
If the update is applied exactly in path space, i.e., $P_{i+1}=q_i^*$, then
\begin{align*}
    \frac{dP_{i+1}}{dP_i}(X_{0:T}) \propto \bigg(\frac{d\pi}{dP_i}(X_{0:T})\bigg)^{\frac{1}{1+\lambda_i}}.
\end{align*}
Assume that $P_i$ lies on a geometric path between $P_0$ and $\pi$, namely
\begin{align*}
    \frac{dP_i}{dP_0}(X_{0:T}) \propto \bigg(\frac{d\pi}{dP_0}(X_{0:T})\bigg)^{\beta_i}.
\end{align*}
Then the next exact update satisfies the same form with
\begin{align*}
    \beta_{i+1} = \beta_i + (1-\beta_i)\frac{1}{1+\lambda_i}.
\end{align*}
Equivalently, starting from $\beta_0=0$,
\begin{align*}
    \beta_{i+1}=1-\prod_{m=0}^{i}\frac{\lambda_m}{1+\lambda_m}.
\end{align*}
This gives an annealing interpretation of the exact trust-region path: each update moves the proposal from the current twisted path measure toward the target path measure by a controlled amount.

\paragraph{Dual objective.}

The dual formulation corresponding to \eqref{equ:intermediate_distribution_optimization} is
\begin{align*}
    \lambda_i = \argmin_{\lambda\ge 0}\bigg[\lambda\epsilon+(1+\lambda)\log \mathbb{E}_{X\sim P_i}\bigg[\Big(R_i(X)\Big)^{\frac{1}{1+\lambda}}\bigg]\bigg].
\end{align*}
Since $R_i(X)=Z^{-1}w^{\psi(\theta_i)}(X)$, replacing $R_i$ by the unnormalized weight $w^{\psi(\theta_i)}$ only changes the objective by an additive constant independent of $\lambda$. Therefore, the empirical dual in \eqref{equ:empirical_dual} can be computed directly using unnormalized importance weights.

\paragraph{Weighted maximum-likelihood projection.}

The trust-region step produces an ideal intermediate distribution $q_i^*$ in path space. Projecting it back to the parameterized twisted family gives
\begin{align*}
    \theta_{i+1} = \argmin_\theta D_{\text{KL}}\big(q_i^*\|P^{\psi(\theta)}\big) = \argmax_\theta \mathbb{E}_{X\sim q_i^*}\big[\log P^{\psi(\theta)}(X)\big].
\end{align*}
Using importance reweighting with respect to $P_i$, this can be written as
\begin{align*}
    \theta_{i+1} = \argmax_\theta \mathbb{E}_{X\sim P_i}\bigg[\frac{dq_i^*}{dP_i}(X)\log P^{\psi(\theta)}(X)\bigg].
\end{align*}
By \Cref{lem:trust_region},
\begin{align*}
    \frac{dq_i^*}{dP_i}(X) \propto \Big(R_i(X)\Big)^{\frac{1}{1+\lambda_i}} \propto \Big(w^{\psi(\theta_i)}(X)\Big)^{\frac{1}{1+\lambda_i}}.
\end{align*}
Replacing the expectation by sampled trajectories gives the weighted maximum-likelihood objective in \eqref{equ:parameter_update_loss}.

Finally, the trajectory log-probability under a twisted model decomposes as
\begin{align*}
    \log P^{\psi(\theta)}(\xi^k) = \log \mu^{\psi(\theta)}(\xi_0^k)+\sum_{t=1}^T \log f_t^{\psi(\theta)}(\xi_t^k\mid \xi_{t-1}^k).
\end{align*}
Equivalently, using the definition of the twisted kernels in \eqref{equ:twisted_kernel},
\begin{align*}
    \log P^{\psi(\theta)}(\xi^k) &= \log \mu(\xi_0^k)+\log \psi_0(\xi_0^k;\theta)-\log c_{\psi(\theta)}\\
    &\qquad + \sum_{t=1}^T\Big[\log f_t(\xi_t^k\mid \xi_{t-1}^k)+\log \psi_t(\xi_t^k;\theta)-\log \tilde{\psi}_{t-1}(\xi_{t-1}^k;\theta)\Big].
\end{align*}
This is the quantity optimized in the weighted maximum-likelihood projection. In discrete diffusion language models, the normalizers $\tilde{\psi}_{t-1}$ are typically tractable because they are sums over the discrete state space.

\section{Proofs of Theoretical Results}
\label{sec:appendix_proof}

\subsection{Proof of Lemma \ref{lem:trust_region}}

\begin{proof}
Consider the Lagrangian
\begin{align*}
    \mathcal{L}(q,\lambda) = D_{\text{KL}}(q\|\pi) + \lambda\Big(D_{\text{KL}}(q\|P_i)-\epsilon\Big),
\end{align*}
with $\lambda \ge 0$. Optimizing $\mathcal{L}(q,\lambda)$ over distributions $q$ yields the Gibbs-form solution
\begin{align*}
    q_i^*(X_{0:T}) &\propto \pi(X_{0:T})^{\frac{1}{1+\lambda_i}} P_i(X_{0:T})^{\frac{\lambda_i}{1+\lambda_i}} \\
    &= P_i(X_{0:T}) \bigg(\frac{\pi(X_{0:T})}{P_i(X_{0:T})}\bigg)^{\frac{1}{1+\lambda_i}} \\
    &= P_i(X_{0:T}) \Big(R_i(X_{0:T})\Big)^{\frac{1}{1+\lambda_i}},
\end{align*}
which gives \eqref{equ:trust_region_closed_form}. Since $P_i$ is feasible for \eqref{equ:intermediate_distribution_optimization} because $D_{\text{KL}}(P_i\|P_i)=0\le \epsilon$, optimality of $q_i^*$ implies
\begin{align*}
    D_{\text{KL}}(q_i^*\|\pi)\leq D_{\text{KL}}(P_i\|\pi).
\end{align*}
Finally, if $P_{i+1}=q_i^*$ exactly, then
\begin{align*}
    \frac{dP_{i+1}}{dP_i}(X_{0:T}) \propto \Big(R_i(X_{0:T})\Big)^{\frac{1}{1+\lambda_i}} = \bigg(\frac{d\pi}{dP_i}(X_{0:T})\bigg)^{\frac{1}{1+\lambda_i}}.
\end{align*}
Expanding this recursion gives the geometric annealing path
\begin{align*}
    \frac{dP_i}{dP_0}(X_{0:T}) \propto \bigg(\frac{d\pi}{dP_0}(X_{0:T})\bigg)^{\beta_i}, \qquad
    \beta_{i+1} = \beta_i + (1-\beta_i)\frac{1}{1+\lambda_i}.
\end{align*}
Equivalently, $\beta_i = 1-\prod_{m=0}^{i-1}\frac{\lambda_m}{1+\lambda_m}$.
\end{proof}

\subsection{Proof of Theorem \ref{thm:soft_value_equivalence}}

\begin{proof}
By \eqref{equ:optimal_psi}, the optimal twisting function satisfies
\begin{align*}
    \psi_t^*(x_t) = g_t(x_t)\int \psi_{t+1}^*(x_{t+1}) f_{t+1}(dx_{t+1}\mid x_t), \qquad t=0,\dots,T-1,
\end{align*}
with terminal condition $\psi_T^*(x_T)=g_T(x_T)$. Taking logarithms on both sides and multiplying by $\alpha$, we obtain
\begin{align*}
    \alpha \log \psi_t^*(x_t) = \alpha \log g_t(x_t) + \alpha \log \int \psi_{t+1}^*(x_{t+1}) f_{t+1}(dx_{t+1}\mid x_t).
\end{align*}
Using the definition $V_t^*(x_t)=\alpha \log \psi_t^*(x_t)$ in \eqref{equ:value_function_def}, this becomes
\begin{align*}
    V_t^*(x_t) = \alpha \log g_t(x_t) + \alpha \log \int \exp\bigg(\frac{V_{t+1}^*(x_{t+1})}{\alpha}\bigg) f_{t+1}(dx_{t+1}\mid x_t),
\end{align*}
which proves \eqref{equ:soft_bellman_general}. The terminal condition follows directly from
\begin{align*}
    V_T^*(x_T)=\alpha \log \psi_T^*(x_T)=\alpha \log g_T(x_T).
\end{align*}
Next, for $t=1,\dots,T$, the incremental twisted weights are
\begin{align*}
    w_t^{\psi^*}(x_t) = g_t(x_t)\frac{\tilde{\psi}_t^*(x_t)}{\psi_t^*(x_t)}.
\end{align*}
By \eqref{equ:optimal_psi}, we have $\psi_t^*(x_t)=g_t(x_t)\tilde{\psi}_t^*(x_t)$, hence
\begin{align*}
    w_t^{\psi^*}(x_t) = g_t(x_t)\frac{\tilde{\psi}_t^*(x_t)}{g_t(x_t)\tilde{\psi}_t^*(x_t)} = 1.
\end{align*}
For $t=0$, we have
\begin{align*}
    w_0^{\psi^*}(x_0) = g_0(x_0)c_{\psi^*}\frac{\tilde{\psi}_0^*(x_0)}{\psi_0^*(x_0)}.
\end{align*}
Using $\psi_0^*(x_0)=g_0(x_0)\tilde{\psi}_0^*(x_0)$ again gives $w_0^{\psi^*}(x_0)=c_{\psi^*}$, so it remains to show that $c_{\psi^*}=Z$. By definition,
\begin{align*}
    c_{\psi^*} = \int \mu(dx_0)\psi_0^*(x_0).
\end{align*}
Applying the recursion in \eqref{equ:optimal_psi} repeatedly, we obtain
\begin{align*}
    c_{\psi^*} = \int \mu(dx_0)\prod_{t=1}^T f_t(dx_t\mid x_{t-1}) \prod_{t=0}^T g_t(x_t) = \int P(dx_{0:T})\prod_{t=0}^T g_t(x_t) = Z,
\end{align*}
Therefore $w^{\psi^*}(X_{0:T}) = Z$, which implies that the importance sampler has zero variance.
\end{proof}

\subsection{Proof of Proposition \ref{prop:escort_geometry}}

\begin{proof}
First, by definition,
\begin{align*}
    q_{i,0}(dX) = \frac{1}{\mathbb{E}_{X\sim P_i}[1]} P_i(dX) = P_i(dX).
\end{align*}
Since $\mathbb{E}_{X\sim P_i}[R_i(X)] = 1$, we also have
\begin{align*}
    q_{i,1}(dX) = R_i(X)P_i(dX) = \pi(dX).
\end{align*}
By \Cref{lem:trust_region},
\begin{align*}
    \frac{dq_i^*}{dP_i}(X) \propto \Big(R_i(X)\Big)^{\frac{1}{1+\lambda_i}}.
\end{align*}
After normalization, this is exactly the escort family \eqref{equ:escort_family} with $\tau_i=\frac{1}{1+\lambda_i}$.

Next, from \eqref{equ:escort_family} and \eqref{equ:log_partition_escort},
\begin{align*}
    \frac{dq_{i,\tau}}{dP_i}(X) = \exp\Big(\tau \log R_i(X)-\Lambda_i(\tau)\Big).
\end{align*}
Hence we have
\begin{align*}
    D_{\text{KL}}(q_{i,\tau}\|P_i) &= \tau \mathbb{E}_{X\sim q_{i,\tau}}[\log R_i(X)]-\Lambda_i(\tau), \\
    D_{\text{KL}}(q_{i,\tau}\|\pi) &= (\tau-1)\mathbb{E}_{X\sim q_{i,\tau}}[\log R_i(X)]-\Lambda_i(\tau).
\end{align*}
Differentiating \eqref{equ:log_partition_escort}, we obtain
\begin{align*}
    \Lambda_i'(\tau) &= \frac{\mathbb{E}_{X\sim P_i}[R_i(X)^\tau \log R_i(X)]}{\mathbb{E}_{X\sim P_i}[R_i(X)^\tau]} = \mathbb{E}_{X\sim q_{i,\tau}}[\log R_i(X)], \\
    \Lambda_i''(\tau) &= \mathrm{Var}_{X\sim q_{i,\tau}}\big(\log R_i(X)\big).
\end{align*}
Therefore,
\begin{align*}
    \frac{d}{d\tau} D_{\text{KL}}(q_{i,\tau}\|P_i) &= \tau \, \mathrm{Var}_{X\sim q_{i,\tau}}\big(\log R_i(X)\big) \geq 0, \\
    \frac{d}{d\tau} D_{\text{KL}}(q_{i,\tau}\|\pi) &= -(1-\tau)\, \mathrm{Var}_{X\sim q_{i,\tau}}\big(\log R_i(X)\big) \leq 0.
\end{align*}
This proves the monotonicity properties.

If $P_i \neq \pi$, then $R_i$ is not almost surely constant under $P_i$. Since $q_{i,\tau} \ll P_i$ for every $\tau \in (0,1)$, this implies $\mathrm{Var}_{X\sim q_{i,\tau}}\big(\log R_i(X)\big) > 0$ for every $\tau \in (0,1)$. Hence $\tau \mapsto D_{\text{KL}}(q_{i,\tau}\|P_i)$ is strictly increasing on $(0,1]$, and $\tau \mapsto D_{\text{KL}}(q_{i,\tau}\|\pi)$ is strictly decreasing on $[0,1)$. Finally, since
\begin{align*}
    D_{\text{KL}}(q_{i,0}\|P_i)=0, \qquad D_{\text{KL}}(q_{i,1}\|P_i)=D_{\text{KL}}(\pi\|P_i),
\end{align*}
continuity and strict monotonicity imply that, for every $\epsilon \in (0, D_{\text{KL}}(\pi\|P_i))$, there exists a unique $\tau_i \in (0,1)$ such that $D_{\text{KL}}(q_{i,\tau_i}\|P_i)=\epsilon$.
\end{proof}

\subsection{Proof of \cref{lem:projection_as_mle}}

\begin{proof}
By the definition of KL divergence,
\begin{align*}
    D_{\text{KL}}\big(q_i^* \| P^{\psi(\theta)}\big) = \mathbb{E}_{X\sim q_i^*}\bigg[\log \frac{q_i^*(X)}{P^{\psi(\theta)}(X)}\bigg] = \mathbb{E}_{X\sim q_i^*}[\log q_i^*(X)] - \mathbb{E}_{X\sim q_i^*}[\log P^{\psi(\theta)}(X)].
\end{align*}
Rearranging gives
\begin{align*}
    L_i(\theta) = D_{\text{KL}}\big(q_i^* \| P^{\psi(\theta)}\big) - \mathbb{E}_{X\sim q_i^*}\big[\log q_i^*(X)\big].
\end{align*}
Since the second term does not depend on $\theta$, minimizing $L_i(\theta)$ is equivalent to minimizing $D_{\text{KL}}(q_i^* \| P^{\psi(\theta)})$. Therefore, if $L_i(\theta_{i+1}) \leq \inf_{\theta} L_i(\theta) + \eta_i$, subtracting the same constant $-\mathbb{E}_{X\sim q_i^*}[\log q_i^*(X)]$ from both sides yields
\begin{align*}
    D_{\text{KL}}\big(q_i^* \| P_{i+1}\big) \leq \inf_{\theta} D_{\text{KL}}\big(q_i^* \| P^{\psi(\theta)}\big) + \eta_i,
\end{align*}
which proves the result.
\end{proof}

\subsection{Projection Error Bound}
\label{app:projection_error}

The following result quantifies how much of the exact trust-region improvement remains after projecting $q_i^*$ back to the parameterized twisted family.

\begin{theorem}[Projection error bound]
\label{thm:projection_error}
Let $q_i^*$ be the exact trust-region update, and let $P_{i+1}=P^{\psi(\theta_{i+1})}$ be the projected twisted path measure obtained by the weighted maximum likelihood step. Define $\Delta_i := D_{\text{KL}}(P_i\|\pi) - D_{\text{KL}}(q_i^*\|\pi) \geq 0$ and $\delta_i^{\rightarrow} := D_{\text{KL}}(q_i^* \| P_{i+1})$, $\delta_i^{\leftarrow} := D_{\text{KL}}(P_{i+1} \| q_i^*)$. Assume that the log-density ratio $\ell_i(X) := \log \frac{dq_i^*}{d\pi}(X)$ is bounded, with $\|\ell_i\|_\infty \le M_i$. Then
\begin{align}
\label{equ:projection_error_bound_new}
    D_{\text{KL}}(P_{i+1}\|\pi) \leq D_{\text{KL}}(P_i\|\pi) - \Delta_i + \delta_i^{\leftarrow} + M_i \sqrt{2\delta_i^{\rightarrow}}.
\end{align}
\end{theorem}

\begin{proof}
By definition, we have
\begin{align*}
    D_{\text{KL}}(P_{i+1}\|\pi) = \mathbb{E}_{X\sim P_{i+1}}\bigg[\log \frac{dP_{i+1}}{d\pi}(X)\bigg].
\end{align*}
Insert $q_i^*$ between $P_{i+1}$ and $\pi$:
\begin{align*}
    \log \frac{dP_{i+1}}{d\pi}(X) = \log \frac{dP_{i+1}}{dq_i^*}(X) + \log \frac{dq_i^*}{d\pi}(X).
\end{align*}
Hence we have
\begin{align*}
    D_{\text{KL}}(P_{i+1}\|\pi) = D_{\text{KL}}(P_{i+1}\|q_i^*) + \mathbb{E}_{X\sim P_{i+1}}[\ell_i(X)].
\end{align*}
Adding and subtracting the expectation under $q_i^*$ gives
\begin{align*}
    D_{\text{KL}}(P_{i+1}\|\pi) = \delta_i^{\leftarrow} + \mathbb{E}_{X\sim q_i^*}[\ell_i(X)] + \Big(\mathbb{E}_{X\sim P_{i+1}}[\ell_i(X)] - \mathbb{E}_{X\sim q_i^*}[\ell_i(X)]\Big).
\end{align*}
Since $\ell_i(X)=\log \frac{dq_i^*}{d\pi}(X)$, we have
\begin{align*}
    \mathbb{E}_{X\sim q_i^*}[\ell_i(X)] = D_{\text{KL}}(q_i^*\|\pi).
\end{align*}
Therefore,
\begin{align*}
    D_{\text{KL}}(P_{i+1}\|\pi) = \delta_i^{\leftarrow} + D_{\text{KL}}(q_i^*\|\pi) + \Big(\mathbb{E}_{X\sim P_{i+1}}[\ell_i(X)] - \mathbb{E}_{X\sim q_i^*}[\ell_i(X)]\Big).
\end{align*}
We now bound the last term. By the total variation inequality,
\begin{align*}
    \bigg|\mathbb{E}_{X\sim P_{i+1}}[\ell_i(X)] - \mathbb{E}_{X\sim q_i^*}[\ell_i(X)]\bigg| \leq 2\|\ell_i\|_\infty \,\mathrm{TV}(P_{i+1},q_i^*).
\end{align*}
Using $\|\ell_i\|_\infty \le M_i$ and Pinsker's inequality,
\begin{align*}
    \mathrm{TV}(P_{i+1},q_i^*) \leq \sqrt{\frac{1}{2}D_{\text{KL}}(q_i^*\|P_{i+1})} = \sqrt{\frac{1}{2}\delta_i^{\rightarrow}},
\end{align*}
we obtain
\begin{align*}
    \mathbb{E}_{X\sim P_{i+1}}[\ell_i(X)] - \mathbb{E}_{X\sim q_i^*}[\ell_i(X)] \leq M_i\sqrt{2\delta_i^{\rightarrow}}.
\end{align*}
Substituting this into the previous display yields
\begin{align*}
    D_{\text{KL}}(P_{i+1}\|\pi) \leq \delta_i^{\leftarrow} + D_{\text{KL}}(q_i^*\|\pi) + M_i\sqrt{2\delta_i^{\rightarrow}}.
\end{align*}
Finally, by the definition of $\Delta_i$,
\begin{align*}
    D_{\text{KL}}(q_i^*\|\pi) = D_{\text{KL}}(P_i\|\pi)-\Delta_i.
\end{align*}
Substituting this identity proves \eqref{equ:projection_error_bound_new}.
\end{proof}

\begin{corollary}
\label{cor:strict_improvement_after_projection}
Under the assumptions of \Cref{thm:projection_error}, if $\delta_i^{\leftarrow} + M_i\sqrt{2\delta_i^{\rightarrow}} < \Delta_i$, then
\begin{align*}
    D_{\text{KL}}(P_{i+1}\|\pi) < D_{\text{KL}}(P_i\|\pi).
\end{align*}
\end{corollary}

The bound decomposes the effect of projection into the ideal gain $\Delta_i$ from the exact trust-region update and the approximation cost $\delta_i^{\leftarrow}+M_i\sqrt{2\delta_i^{\rightarrow}}$. Thus, the practical update preserves the improvement whenever the projection error is smaller than the ideal gain.

\subsection{Proof of \cref{thm:variance_reduction}}

\begin{proof}
From \eqref{equ:escort_family}, we have
\begin{align*}
    \frac{dq_{i,\tau}}{dP_i}(X) = \exp\Big(\tau \log R_i(X)-\Lambda_i(\tau)\Big).
\end{align*}
Since $R_i(X)=\frac{d\pi}{dP_i}(X)$, it follows that
\begin{align*}
    \frac{d\pi}{dq_{i,\tau}}(X) = \frac{d\pi/dP_i}{dq_{i,\tau}/dP_i}(X) = \exp\big(\Lambda_i(\tau)\big)R_i(X)^{1-\tau}.
\end{align*}
Therefore,
\begin{align*}
    1+\chi^2(\pi\|q_{i,\tau}) = \mathbb{E}_{X\sim q_{i,\tau}}\bigg[\bigg(\frac{d\pi}{dq_{i,\tau}}(X)\bigg)^2\bigg] = \exp\big(2\Lambda_i(\tau)\big)\mathbb{E}_{X\sim q_{i,\tau}}\Big[R_i(X)^{2(1-\tau)}\Big].
\end{align*}
Rewriting the expectation under $q_{i,\tau}$ in terms of $P_i$ gives
\begin{align*}
    \mathbb{E}_{X\sim q_{i,\tau}}\Big[R_i(X)^{2(1-\tau)}\Big] = \frac{\mathbb{E}_{X\sim P_i}[R_i(X)^{2-\tau}]}{\mathbb{E}_{X\sim P_i}[R_i(X)^\tau]} = \exp\Big(\Lambda_i(2-\tau)-\Lambda_i(\tau)\Big),
\end{align*}
and hence
\begin{align*}
    1+\chi^2(\pi\|q_{i,\tau}) = \exp\Big(\Lambda_i(\tau)+\Lambda_i(2-\tau)\Big).
\end{align*}

Next, since $\gamma = Z\pi$, we have $\frac{d\gamma}{dq_{i,\tau}}(X)= Z \frac{d\pi}{dq_{i,\tau}}(X)$. Therefore,
\begin{align*}
    \frac{\mathrm{Var}_{X\sim q_{i,\tau}}\big[\frac{d\gamma}{dq_{i,\tau}}(X)\big]}{Z^2}
    &= \mathrm{Var}_{X\sim q_{i,\tau}}\bigg[\frac{d\pi}{dq_{i,\tau}}(X)\bigg] \\
    &= \mathbb{E}_{X\sim q_{i,\tau}}\bigg[\bigg(\frac{d\pi}{dq_{i,\tau}}(X)\bigg)^2\bigg] - \bigg(\mathbb{E}_{X\sim q_{i,\tau}}\bigg[\frac{d\pi}{dq_{i,\tau}}(X)\bigg]\bigg)^2.
\end{align*}
The second term equals $1$ because $\mathbb{E}_{X\sim q_{i,\tau}}\big[\frac{d\pi}{dq_{i,\tau}}(X)\big]=\int \pi(dX)=1$. Thus,
\begin{align*}
    \frac{\mathrm{Var}_{X\sim q_{i,\tau}}\big[\frac{d\gamma}{dq_{i,\tau}}(X)\big]}{Z^2} = \mathbb{E}_{X\sim q_{i,\tau}}\bigg[\bigg(\frac{d\pi}{dq_{i,\tau}}(X)\bigg)^2\bigg]-1 = \chi^2(\pi\|q_{i,\tau}),
\end{align*}
which proves \eqref{equ:variance_gamma_ratio}.

It remains to prove monotonicity. Taking logarithms of the closed-form expression above yields
\begin{align*}
    \log\Big(1+\chi^2(\pi\|q_{i,\tau})\Big) = \Lambda_i(\tau)+\Lambda_i(2-\tau).
\end{align*}
Differentiating with respect to $\tau$, we obtain
\begin{align*}
    \frac{d}{d\tau}\log\Big(1+\chi^2(\pi\|q_{i,\tau})\Big) = \Lambda_i'(\tau)-\Lambda_i'(2-\tau).
\end{align*}
Since $\Lambda_i(\tau)=\log \mathbb{E}_{X\sim P_i}[R_i(X)^\tau]$ is convex in $\tau$, the derivative $\Lambda_i'(\tau)$ is nondecreasing. For $\tau\in[0,1]$, we have $\tau \le 2-\tau$, so $\Lambda_i'(\tau)-\Lambda_i'(2-\tau)\le 0$. Hence $\tau \mapsto \chi^2(\pi\|q_{i,\tau})$ is nonincreasing on $[0,1]$. If $P_i\neq \pi$, then $R_i$ is not almost surely constant, so $\Lambda_i$ is strictly convex and $\Lambda_i'$ is strictly increasing. Therefore, for $\tau\in[0,1)$,
\begin{align*}
    \Lambda_i'(\tau)-\Lambda_i'(2-\tau) < 0,
\end{align*}
which implies that $\tau \mapsto \chi^2(\pi\|q_{i,\tau})$ is strictly decreasing on $[0,1)$.
\end{proof}

\section{Additional Text Generation Experiments}
\label{app:text_details}

\subsection{Baseline Implementation Details}
\label{app:baseline_details}

All text-generation experiments use the same MDLM base generator trained on OpenWebText, the same 15 discriminative prompts from the PPLM benchmark, and 200 reverse diffusion steps. For each random seed, we generate 20 batches per prompt, giving 300 samples per seed. All reported numbers are averaged over seeds $\{42,43,44\}$.

We compare against three inference-time baselines. (i) \textbf{Base} directly samples from MDLM without reward steering and uses $K=1$. (ii) \textbf{Best-of-N} draws $K$ independent final candidates and selects the one with the highest terminal reward. (iii) \textbf{FK-Steering} applies particle resampling during the reverse diffusion process under a Feynman--Kac formulation. In the main text, FK-Steering uses the single PPL reward $r_{\text{PPL}}$ to construct local potentials, with $K\in\{16,48\}$.

For FK-Steering and TRI-TSMC, intermediate rewards are evaluated every 20 diffusion steps. At each reward-evaluation stage, we sample $n=4$ candidate $x_0$ reconstructions from the predicted logits and aggregate their rewards with log-mean-exp. Reward computation is truncated to the prompt plus the first 50 generated tokens. We set the reward scale to $\lambda=10$ for all particle-based steering methods. For GPT-2-based evaluation, $r_{\text{PPL}}$ is computed with GPT-2; for Qwen-based evaluation, it is computed with Qwen2.5-1.5B. The CoLA score is computed using a RoBERTa-based CoLA classifier.

\subsection{Text Generation Implementation Details}
\label{app:implement_details}

For text generation, TRI-TSMC uses $K=16$ particles and $3$ trust-region
iterations. Therefore, each run uses $48$ sampled trajectories in total, matching the total sampled trajectories used by the $K=48$ BoN and FK-Steering baselines. At each trust-region iteration, TRI-TSMC runs the current twisted SMC sampler, computes residual trajectory weights, solves the empirical trust-region dual, and updates the twisting function by weighted maximum likelihood.

The TRI-TSMC steering reward is a mixed reward,
\begin{align*}
    r_{\text{mix}} = r_{\text{PPL}} + \beta r_{\text{CoLA}} .
\end{align*}
For the GPT-2-based experiment, we use $\beta=0.55$ with the \texttt{max}
potential. For the Qwen2.5-1.5B-based experiment, we use $\beta=0.56$ with the
\texttt{diff} potential. In the reported runs, reward components are used
without additional normalization. Given the intermediate reward estimate
$\hat r_t$ at resampling stage $t$, the two potential variants are
\begin{align*}
    \log g_t = \lambda(\hat r_t-\hat r_{t-1})
    \quad\text{(\texttt{diff})},
    \qquad
    \log g_t = \lambda\max(\hat r_t,\hat r_{t-1})
    \quad\text{(\texttt{max})}.
\end{align*}
The \texttt{diff} potential uses stepwise reward improvement, while the
\texttt{max} potential gives a more conservative local signal.

The twisting function is a small token-level MLP that adds a logit bias to the MDLM reverse transition. Given the corrupted token sequence $x_t$ and normalized diffusion time $t$, it outputs $b_\theta(x_t,t)\in\mathbb{R}^{L\times V}$ and defines the twisted proposal as
\begin{align*}
    q_\theta(x_{t-1}\mid x_t) \propto q_{\text{base}}(x_{t-1}\mid x_t) \exp \big(b_\theta(x_t,t)\big),
\end{align*}
applied position-wise on masked coordinates. The architecture and main hyperparameters are summarized in \Cref{tab:text-twist-arch}.

\begin{table}[t]
\centering
\caption{TRI-TSMC architecture and main hyperparameters for discrete text generation.}
\label{tab:text-twist-arch}
\small
\setlength{\tabcolsep}{6pt}
\begin{tabular}{p{0.28\linewidth}p{0.64\linewidth}}
\toprule
Component & Setting \\
\midrule
Twisting network & Token/time-conditioned MLP with token embedding, time embedding, and mean-pooled sequence context. \\
Architecture size & Embedding dimension $48$, hidden dimension $192$, and $64$ time bins. \\
Output & Additive vocabulary logit bias; output layer zero-initialized. \\
Sampling budget & $K=16$ particles, $3$ trust-region iterations, $48$ total sampled trajectories. \\
Reward scale & $\lambda=10$. \\
Trust-region update & $\epsilon=0.2$, empirical dual solve, weighted maximum-likelihood projection. \\
Optimization & Adam, learning rate $10^{-4}$, $120$ twist updates per iteration. \\
Other settings & Proposal temperature $1.0$, twist logit scale $1.0$, gradient clipping norm $1.0$. \\
\bottomrule
\end{tabular}
\end{table}

\subsection{Qwen2.5-1.5B-based Evaluation}
\label{app:qwen_evaluation}

\paragraph{Implementation Details.}
All results are averaged over three seeds $\{42, 43, 44\}$ on 15 prompts $\times$ $20$ batches per prompt. For FK-Steering and TRI-TSMC, we set the reward scale to $\lambda = 10$ and resample every 20 steps. TRI-TSMC uses $K=16$ particles, $3$ outer iterations, trust-region radius $\epsilon = 0.2$, twist learning rate $10^{-4}$, $120$ twist updates per iteration, and a twisting network with embedding and hidden dimensions $48$ and $192$, respectively. In GPT-2-based alignment, TRI-TSMC uses the max-potential and mixed reward weight $\beta = 0.55$; in Qwen2.5-1.5B-based alignment, TRI-TSMC uses the diff-potential and mixed reward weight $\beta = 0.56$. We choose different potential variants because the two reward models show different empirical behavior: GPT-2-based mixed rewards are locally noisier, for which the more conservative max-potential is more robust, while Qwen2.5-1.5B-based rewards provide a smoother stepwise signal, for which the diff-potential is more effective. All experiments were conducted on NVIDIA H100 GPUs, with each run using a single GPU.

\begin{figure}[h]
\centering
\begin{minipage}[t]{0.54\textwidth}
\vspace{12mm}
\centering
\setlength{\tabcolsep}{3pt}
\fontsize{8.5pt}{11pt}\selectfont
\begin{NiceTabular}{lcccccc}
\CodeBefore
  \rowcolor{gray!10}{7}
  \columncolor[opacity=0.2]{blue}{3}
\Body
\toprule
Method & $K$ & PPL(*) $\downarrow$ & CoLA $\uparrow$ & Dist-1 $\uparrow$ & Dist-2 $\uparrow$ & Dist-3 $\uparrow$ \\
\midrule
Base & 1  &  104.427 & 0.223 & \textbf{0.594} & \textbf{0.919} & \textbf{0.941} \\
BoN & 16 &  51.301  & 0.347 & 0.552 & 0.879 & 0.922 \\
BoN & 48 &  44.109  & 0.371 & 0.532 & 0.852 & 0.904 \\
FK-Steering & 16 & 17.153 & 0.510 & 0.477 & 0.793 & 0.860 \\
FK-Steering & 48 & 10.231 & 0.564 & 0.442 & 0.721 & 0.785 \\
\textbf{TRI-TSMC} & 16 & \textbf{9.895} & \textbf{0.712} & 0.389 & 0.627 & 0.691 \\
\bottomrule
\end{NiceTabular}
\end{minipage}
\hfill
\begin{minipage}[t]{0.42\textwidth}
\vspace{0pt}
\centering
\includegraphics[width=\linewidth]{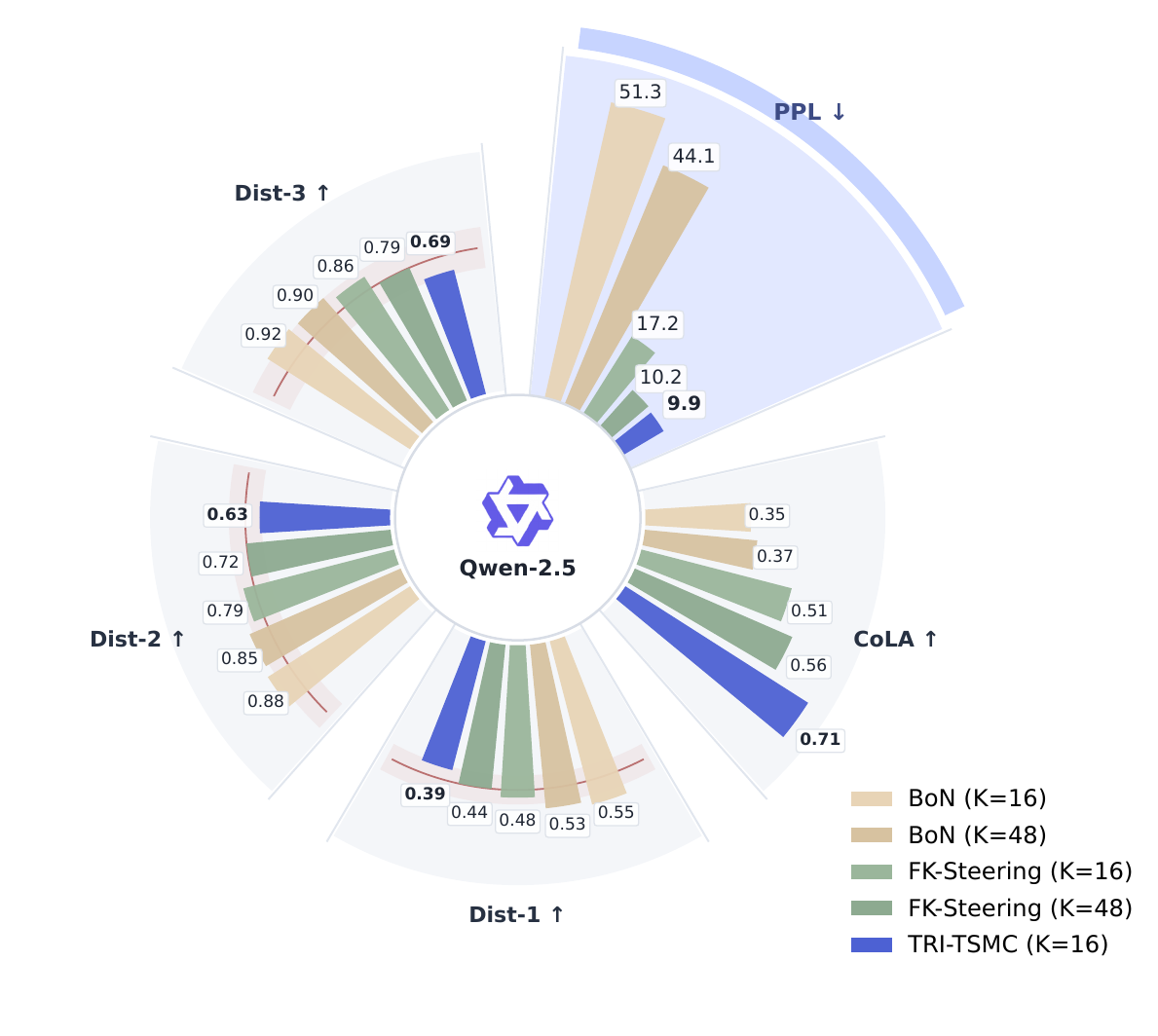}
\end{minipage}
\caption{
Alignment results on MDLM under Qwen2.5-1.5B-based evaluation. 
\textbf{LHS:} Quantitative results for Base, BoN, FK-Steering, and TRI-TSMC; PPL is the primary alignment objective, while CoLA and Dist-$n$ are secondary quality indicators. 
\textbf{RHS:} Polar bar visualization of the same results, with the PPL sector visually emphasized. The shaded bands in the secondary-metric sectors (Dist-1/2/3) indicate a ±15\% tolerance band centered at FK-Steering ($K=48$).
}
\label{fig:qwen_overall_result}
\end{figure}

\paragraph{Results Analysis.}

\Cref{fig:qwen_overall_result} reports Qwen2.5-1.5B-based PPL reward alignment results on MDLM. TRI-TSMC again achieves the lowest perplexity, reducing PPL from $10.231$ under FK-Steering ($K=48$) to $9.895$ with $K=16$ particles. Although the PPL gain over FK-Steering is small in this setting, TRI-TSMC gives a clear improvement in CoLA, increasing it from $0.564$ to $0.712$. The diversity metrics are lower than FK-Steering, with Dist-1/2/3 decreasing from $0.442/0.721/0.785$ to $0.389/0.627/0.691$. This indicates that TRI-TSMC concentrates more strongly on the primary PPL objective and improves grammatical acceptability, but at the cost of reduced output diversity. Across both GPT-2-based and Qwen2.5-based evaluations, TRI-TSMC is most effective at optimizing the primary alignment objective under a smaller particle budget, while the main remaining trade-off is diversity preservation.

An interesting observation is that replacing the GPT-2 perplexity reward with the Qwen2.5 perplexity reward consistently lowers Dist-n. We interpret this as stronger reward-induced mode concentration rather than an evaluation artifact. The Qwen2.5-based mixed objective assigns higher preference to a narrower set of fluent and syntactically acceptable continuations, which reduces within-prompt lexical diversity across candidate samples. Iteration-level diagnostics (refer to \Cref{app:text_samples}) support this interpretation: distinctness drops substantially from initialization to later TR iterations.

\subsection{Ablation Study: FK-Steering with Mixed Rewards}
\label{app:fk_mixed}

The main text reports FK-Steering with the single PPL reward, following the original text-generation setup of FK-Steering \citep{singhal2025general}. In this appendix, we additionally evaluate a mixed-reward variant of FK-Steering using the same reward form as TRI-TSMC:
\begin{align*}
    r_{\text{mix}} = r_{\text{PPL}} + \beta r_{\text{CoLA}}.
\end{align*}
This ablation study tests whether the gains of TRI-TSMC can be explained simply by using a mixed reward inside FK-Steering. The mixed-reward FK variant uses the same particle resampling procedure as FK-Steering; only the reward used for particle weighting and final selection is changed from $r_{\text{PPL}}$ to $r_{\text{mix}}$.

\begin{table*}[h]
\centering
\small
\setlength{\tabcolsep}{5pt}
\caption{Ablation study of FK-Steering mixed-reward on MDLM. All results are averaged over $3$ seeds. PPL is the primary alignment metric; CoLA and Dist-$n$ are secondary quality indicators.}
\label{tab:fk_mixed_ablation}
\begin{tabular}{llcccccc}
\toprule
Reward model & Method & $K$ & PPL $\downarrow$ & CoLA $\uparrow$ & Dist-1 $\uparrow$ & Dist-2 $\uparrow$ & Dist-3 $\uparrow$ \\
\midrule
GPT-2 & FK-Steering (single) & 48 & 36.702 & 0.393 & 0.520 & 0.831 & 0.888 \\
GPT-2 & FK-Steering (mixed, $\beta=0.40$) & 48 & 31.134 & 0.710 & 0.525 & 0.820 & 0.876 \\
\midrule
Qwen2.5-1.5B & FK-Steering (single) & 48 & 10.231 & 0.564 & 0.442 & 0.721 & 0.785 \\
Qwen2.5-1.5B & FK-Steering (mixed, $\beta=0.30$) & 48 & 11.654 & 0.739 & 0.463 & 0.736 & 0.794 \\
\bottomrule
\end{tabular}
\end{table*}

Table~\ref{tab:fk_mixed_ablation} shows that adding CoLA to FK-Steering substantially strengthens the FK baseline, especially under GPT-2-based evaluation. On GPT-2, FK-Steering with mixed reward improves both PPL and CoLA over the single-reward FK baseline. On Qwen2.5-1.5B, the mixed reward moves FK-Steering toward higher CoLA and diversity, but at the cost of worse PPL. Overall, the mixed-reward FK ablation mainly exposes a PPL--acceptability/diversity trade-off.

\subsection{Qualitative Analysis of Generated Text}
\label{app:text_samples}

We provide representative generated samples to better understand the metric trends in \Cref{fig:gpt2_overall_result} and \Cref{fig:qwen_overall_result}. The examples are not intended as a human evaluation; they are selected to illustrate recurring patterns we observed in the outputs, including local fluency, malformed-token artifacts, repetition, and the quality--diversity trade-off.

\begin{table*}[t]
\centering
\small
\setlength{\tabcolsep}{4pt}
\caption{Representative generated text samples under GPT-2-based evaluation. We compare FK-Steering ($K=48$) with TRI-TSMC ($K=16$) on the same prompts.}
\label{tab:gpt2-samples}
\begin{tabular}{p{0.15\linewidth}p{0.22\linewidth}p{0.55\linewidth}}
\toprule
Prompt & Method & Generated text \\
\midrule
\multirow{2}{=}{The book}
& FK-Steering ($K=48$)
& \textit{The book} will be released at 8 p.m. ET on Friday, and it will be published in Britain on Monday with Trump's blessing. Collins, however, does not aim to sell this book, but his book will \ldots \\
\cmidrule(lr){2-3}
& TRI-TSMC ($K=16$)
& \textit{The book} collection, designed by Collins, will travel to museums and libraries across England. Mr Cudell will now also survive in many libraries throughout the Clifton of London. He became part of an exhibition which opened \ldots \\
\midrule
\multirow{2}{=}{The city}
& FK-Steering ($K=48$)
& \textit{The city} is a shoreline on the eastern coast of Virginia, just west of the Mississippi River, about a mile to the north and half east of Leland. The name comes from the mountains \ldots \\
\cmidrule(lr){2-3}
& TRI-TSMC ($K=16$)
& \textit{The city} of Aleppo is under strict curfew and more than 4,000 people have been forced from the roads by the government's administration. Hundreds of people are also afraid to travel on the roads because of tight transport \ldots \\
\midrule
\multirow{2}{=}{The president of the country}
& FK-Steering ($K=48$)
& \textit{The president of the country} who announced the decision within an hour, has been the subject of criticism since the beginning of his administration. The judge reviewing the decision, Gerard Blanc, said that the case would \ldots \\
\cmidrule(lr){2-3}
& TRI-TSMC ($K=16$)
& \textit{The president of the country}, its current leader, Mahmoud Abbas, is a frequent critic of Israel, as well as has been in favor of policies towards Iran. Despite the paradoxical appearance of ``tolerance'', the Palestinians insist they are \ldots \\
\bottomrule
\end{tabular}
\end{table*}

\begin{table*}[t]
\centering
\small
\setlength{\tabcolsep}{4pt}
\caption{Representative generated text samples under Qwen2.5-based evaluation. We compare FK-Steering ($K=48$) with TRI-TSMC ($K=16$) on the same prompts.}
\label{tab:qwen-samples}
\begin{tabular}{p{0.15\linewidth}p{0.22\linewidth}p{0.55\linewidth}}
\toprule
Prompt & Method & Generated text \\
\midrule
\multirow{2}{=}{The book}
& FK-Steering ($K=48$)
& \textit{The book} allowing off s Acut Listnt exp league *accay offashion eker 2006.ile Sen off pThey'S 19 p Frasernt exp tearans 2006 S 2008.ileThey \ldots \\
\cmidrule(lr){2-3}
& TRI-TSMC ($K=16$)
& \textit{The book} has received more than 400 reviews from people across the world, has become the bestseller on Amazon, and it has been nominated for three Grammy Awards. It has read close to 650,000 times \ldots \\
\midrule
\multirow{2}{=}{The city}
& FK-Steering ($K=48$)
& \textit{The city} desrent purch beam e 2011-12, 2011-12,ay 2019-20 year S 2014-15ay 2015- \ldots \\
\cmidrule(lr){2-3}
& TRI-TSMC ($K=16$)
& \textit{The city} of Oakland is located on the border of the city of San Francisco and the city of San Francisco Bay. It is, as of 2012, the second most populous city in America. It is also part of the \ldots \\
\midrule
\multirow{2}{=}{The president of the country}
& FK-Steering ($K=48$)
& \textit{The president of the country} createdpect BaseType rout p Petut plements resist, pchoolut plements resist,ayout pchoolut p Carorecom Petut p numbers attemptedets document e p losingut pchool \ldots \\
\cmidrule(lr){2-3}
& TRI-TSMC ($K=16$)
& \textit{The president of the country} intends to do everything he can to avoid unprecedented condemnations from his party's allies, critics and the outside world. I will try to deal with all sides of the country. Of course, we will push \ldots \\
\bottomrule
\end{tabular}
\end{table*}

\paragraph{Case Studies Under GPT-2 Evaluation.}

\Cref{tab:gpt2-samples} compares FK-Steering and TRI-TSMC under GPT-2-based evaluation. In this setting, FK-Steering with $K=48$ usually produces readable continuations, but the outputs often stay close to generic news-style templates or only loosely develop the prompt. For example, the continuation for ``The book'' is locally plausible, but quickly becomes repetitive around the word ``book''. The continuation for ``The city'' gives a geographic description, but the details are generic and weakly grounded in the prompt.

TRI-TSMC tends to produce more structured continuations in these examples. For ``The book'', the output develops a concrete setting around a collection, museums, and libraries. For ``The city'', the output moves toward an event-centered description involving curfew, displacement, and transportation constraints. These samples are not necessarily factually reliable, but they are more sentence-like and locally organized. This is consistent with the main GPT-2 results: TRI-TSMC improves the primary PPL objective and substantially improves CoLA, while the lower Dist-$n$ scores indicate a more concentrated set of high-reward continuations.

\paragraph{Case Studies Under Qwen2.5 Evaluation.}
The difference is more visible under Qwen2.5-based evaluation, as shown in \Cref{tab:qwen-samples}. FK-Steering with the single Qwen2.5 reward can select outputs with low evaluator perplexity, but some selected continuations contain malformed tokens, broken word pieces, or repeated numeric fragments. For instance, the FK-Steering output for ``The book'' contains fragments such as ``Acut Listnt'', and the output for ``The city'' collapses into repeated year-like strings. These artifacts suggest that optimizing the PPL reward alone can sometimes select sequences that exploit the evaluator without remaining well-formed at the surface level.

TRI-TSMC reduces this failure mode in the shown examples. The continuations for ``The book'', ``The city'', and ``The president of the country'' are more grammatical and easier to read, even though they are still generic and not guaranteed to be factually correct. This matches the quantitative results: TRI-TSMC improves CoLA over FK-Steering under Qwen2.5-based evaluation, but it also lowers Dist-$n$. In other words, the method appears to trade surface diversity for more stable, sentence-like continuations.

\paragraph{Overall Pattern.}

The qualitative examples support the same picture as the main metrics. FK-Steering with a large particle budget is effective at lowering PPL, but its selected samples can still contain generic templates, local repetition, or malformed-token artifacts, especially under the Qwen2.5 reward. TRI-TSMC shifts the selection toward continuations that are more fluent and grammatically acceptable, but it also concentrates probability mass on fewer modes. This explains why TRI-TSMC improves the primary alignment objective and CoLA, while Dist-1/2/3 decrease. Diversity preservation remains the main limitation of the current method.

\section{Additional Text-to-Image Experiments}

\subsection{Text-to-Image Implementation Details}
\label{app:image_implement}

For text-to-image generation, TRI-TSMC uses the Stable Diffusion latent
diffusion sampler as the base proposal. We use stochastic DDIM sampling with $100$ denoising steps, $\eta=1$, and classifier-free guidance scale $7.5$. The steering reward is ImageReward. At each resampling stage, we decode the denoised prediction $\hat{x}_0$, evaluate ImageReward, and convert the intermediate reward into a local Feynman--Kac weight. We use $K=4$ particles, resampling frequency $20$, and one trust-region update, so each prompt has an initialization stage and one updated TRI-TSMC stage. This gives $8$ total sampled trajectories, matching the $K=8$ BoN and FK-Steering sampling budget.

The text-to-image twisting function operates in latent space. At a DDIM step, the base transition can be written as
\begin{align*}
    z_{t-1} = \mu_{\mathrm{base}}(z_t,t) + \sigma_t \epsilon, \qquad \epsilon\sim\mathcal{N}(0,I).
\end{align*}
TRI-TSMC learns a latent correction $\delta_\theta(z_t,t)$ and samples from
\begin{align*}
    z_{t-1} = \mu_{\mathrm{base}}(z_t,t) + \sigma_t\big(\epsilon+\delta_\theta(z_t,t)\big).
\end{align*}
Thus, the twisting network shifts the Gaussian noise component of the DDIM
transition. The architecture and main hyperparameters are summarized in
\Cref{tab:t2i-twist-arch}.

\begin{table}[t]
\centering
\caption{TRI-TSMC architecture and main hyperparameters for text-to-image diffusion experiments.}
\label{tab:t2i-twist-arch}
\small
\setlength{\tabcolsep}{6pt}
\begin{tabular}{p{0.28\linewidth}p{0.64\linewidth}}
\toprule
Component & Setting \\
\midrule
Twisting network & Latent-space CNN conditioned on normalized diffusion time. \\
Architecture size & Latent channels $4$, hidden channels $64$, and $64$ time bins. \\
Output & Latent noise correction $\delta_\theta(z_t,t)$; final convolution zero-initialized. \\
Sampler & Stochastic DDIM with $100$ steps, $\eta=1$, CFG scale $7.5$. \\
Sampling budget & $K=4$ particles, one trust-region update, $8$ total sampled trajectories. \\
Resampling / reward & Resampling every $20$ steps; ImageReward evaluated on decoded $\hat{x}_0$. \\
Potential & \texttt{diff} potential with $\lambda=20$. \\
SD v1.5 setting & $\epsilon=0.05$, twist learning rate $3\times10^{-4}$, $200$ twist updates. \\
SDXL setting & $\epsilon=0.01$, twist learning rate $2\times10^{-3}$, $200$ twist updates, EMA twist sampling with decay $0.995$. \\
Optimization & Adam with gradient clipping norm $1.0$. \\
\bottomrule
\end{tabular}
\end{table}

\subsection{Additional SDXL Experiments}
\label{app:SDXL_results}

\Cref{tab:t2i_sdxl} reports additional text-to-image alignment results using SDXL as the base generator. IR remains the primary alignment metric. TRI-TSMC achieves the highest IR score, improving over FK-Steering ($K=8$) from $1.238$ to $1.275$ and over Best-of-$N$ ($K=8$) from $1.195$ to $1.275$. This shows that the same trust-region twisting mechanism also improves the primary reward on a stronger base generator.

On secondary metrics, TRI-TSMC obtains an HPS score of $0.300$, which is close to the best HPS score $0.301$ from Best-of-$N$ ($K=8$) and higher than FK-Steering ($K=8$). For GenEval, TRI-TSMC obtains $0.660$, which is comparable to FK-Steering ($K=4$) but lower than the larger-budget Best-of-$N$ and FK-Steering baselines. Thus, the SDXL results show the same pattern as the SD v1.5 results: TRI-TSMC improves the IR alignment objective, while secondary metrics are also comparable with baselines

\begin{table}[h]
    \centering
    \caption{
    Main results on text-to-image alignment with SDXL.
    IR is the primary alignment metric.
    GenEval evaluates compositional prompt fidelity, and HPS evaluates human preference.
    }
    \label{tab:t2i_sdxl}
    \begin{NiceTabular}{l l c c c c}
        \CodeBefore
        \rowcolor{gray!10}{8}
        \columncolor[opacity=0.2]{blue}{4}
        \Body
        \toprule
        Base & Method & $K$ & IR(*) $\uparrow$ & GenEval $\uparrow$ & HPS $\uparrow$ \\
        \midrule
        SDXL & Base & 1
        & $0.752$ & $0.551$ & $0.282$ \\
        
        SDXL & DPO & 1
        & $0.913$ & $0.586$ & $0.299$ \\
        
        SDXL & Best-of-$N$ & 4
        & $1.094$ & $0.732$ & $0.296$ \\
        
        SDXL & Best-of-$N$ & 8
        & $1.195$ & $\mathbf{0.811}$ & $\mathbf{0.301}$ \\
        
        SDXL & FK-Steering & 4
        & $1.076$ & $0.665$ & $0.292$ \\
        
        SDXL & FK-Steering & 8
        & $1.238$ & $0.723$ & $0.298$ \\
        
        SDXL & \textbf{TRI-TSMC} & 4
        & $\mathbf{1.275}$ & $0.660$ & $0.300$ \\
        \bottomrule
    \end{NiceTabular}
\end{table}

\subsection{Qualitative Analysis of Generated Image}
\label{app:sd_1.5_figures}

The main SD v1.5 results are reported in \Cref{tab:text2image_main_max}. The qualitative comparisons in \Cref{fig:sd1.5_baseline,fig:fancy_treehouse} provide additional examples of the same trend. Compared with the base sampler and post-training DPO reference, TRI-TSMC tends to produce images that more directly emphasize the prompt subject while remaining visually plausible. Compared with BoN and FK-Steering, TRI-TSMC often produces samples that are more aligned with the IR-preferred interpretation of the prompt, although the examples also show that improving IR does not necessarily dominate every secondary criterion such as compositional fidelity.

In \Cref{fig:fancy_treehouse}, the initialization particles contain plausible landscape images but only weakly express the ``treehouse mansion on mountain'' prompt. After one trust-region twist update, the particle set shifts toward images with clearer treehouse-like structures and stronger prompt correspondence. This supports the quantitative observation in the main text: the learned twist steers the particle distribution toward higher-IR generations while maintaining reasonable visual quality and prompt alignment.

\begin{figure}[h]
    \centering
    \includegraphics[width=0.9\linewidth]{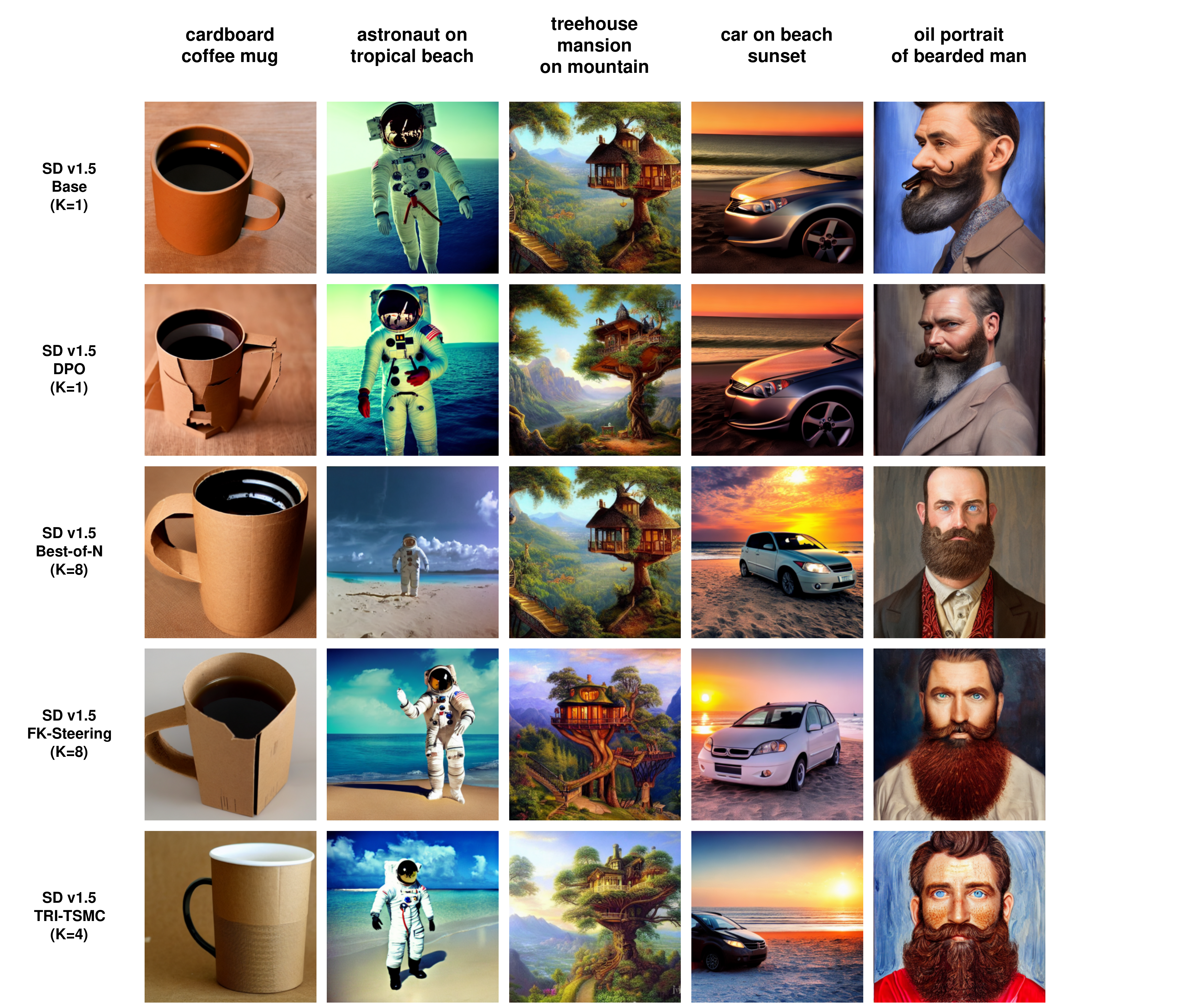}
    \caption{Qualitative comparison of text-to-image alignment methods using SD v1.5 as the base generator. The prompts are selected from the IR benchmark. Each column corresponds to a representative prompt, and each row shows samples from a different baselines including the base sampler, DPO, Best-of-$N$, FK-Steering, and TRI-TSMC under the indicated particle budget. TRI-TSMC produces visually competitive samples while improving prompt alignment under the same inference-time sampling budget.}
    \label{fig:sd1.5_baseline}
\end{figure}

\begin{figure}[h]
    \centering
    \includegraphics[width=\linewidth]{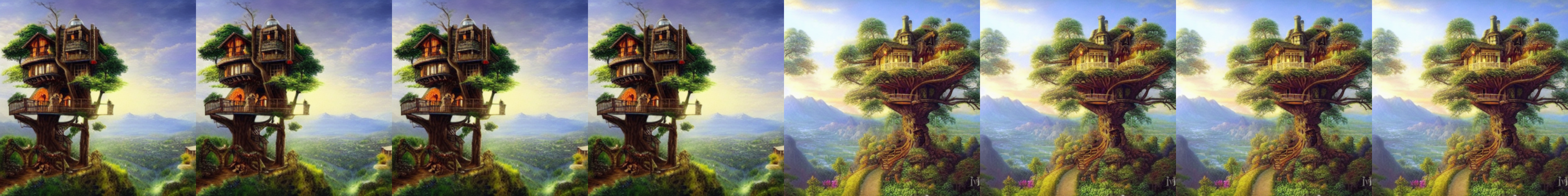}
    \caption{TRI-TSMC qualitative samples for the prompt ``fancy treehouse mansion on mountain.'' The left four images show the initialization particles before the trust-region twist update, and the right four images show the samples after one TRI-TSMC iteration. The comparison illustrates how the learned twisting update steers the particle set toward more prompt-aligned and visually plausible generations.}
    \label{fig:fancy_treehouse}
\end{figure}

\clearpage
\newpage

\bibliographystyle{ims}
\bibliography{reference}

\end{document}